\pdfoutput=1
\documentclass[11pt]{article}

\usepackage{acl}

\usepackage{times}
\usepackage{latexsym}

\usepackage{graphicx}
\usepackage{amsmath}
\usepackage{booktabs}
\usepackage{amsfonts,amssymb}
\usepackage{amssymb}
\usepackage{url}
\usepackage{enumitem}
\usepackage{multirow}
\usepackage{amsthm}
\usepackage{tikz}
\usepackage{tikz-dependency}

\newtheorem{theorem}{Theorem}
\newtheorem{lemma}{Lemma}

\usepackage[T1]{fontenc}
\usepackage[utf8]{inputenc}

\usepackage{microtype}

\title{FormNet: Structural Encoding beyond Sequential Modeling in \\ Form Document Information Extraction}

\author{
Chen-Yu Lee\textsuperscript{$\dagger$}, 
Chun-Liang Li\textsuperscript{$\dagger$},
Timothy Dozat\textsuperscript{$\ddagger$},
Vincent Perot\textsuperscript{$\ddagger$},
Guolong Su\textsuperscript{$\ddagger$}, \\
\textbf{
Nan Hua\textsuperscript{$\ddagger$},
Joshua Ainslie\textsuperscript{$\ddagger$},
Renshen Wang\textsuperscript{$\ddagger$},  
Yasuhisa Fujii\textsuperscript{$\ddagger$},
Tomas Pfister\textsuperscript{$\dagger$}} \\
\textsuperscript{$\dagger$}Google Cloud AI Research, \textsuperscript{$\ddagger$}Google Research \\
\texttt{\footnotesize{\{chenyulee, chunliang, tdozat, vperot, gsu, }} \\
\texttt{\footnotesize{nhua, jainslie, rewang, yasuhisaf, tpfister\}}@google.com} \\
}

\begin{document}
\maketitle
\begin{abstract}
Sequence modeling has demonstrated state-of-the-art performance on natural language and document understanding tasks. 
However, it is challenging to correctly serialize tokens in form-like documents in practice due to their variety of layout patterns. We propose FormNet, a structure-aware sequence model to mitigate the suboptimal serialization of forms.
First, we design \emph{Rich Attention} that leverages the spatial relationship between tokens in a form for more precise attention score calculation. Second, we construct \emph{Super-Tokens} for each word by embedding representations from their neighboring tokens through graph convolutions. FormNet therefore explicitly recovers local syntactic information that may have been lost during serialization. In experiments, FormNet outperforms existing methods with a more compact model size and less pre-training data, establishing new state-of-the-art performance on CORD, FUNSD and Payment benchmarks.
\end{abstract}

\section{Introduction}
Form-like document understanding is a surging research topic because of its practical applications in automating the process of extracting and organizing valuable text data sources such as marketing documents, advertisements and receipts.

Typical documents are represented using natural languages; understanding articles or web content~\cite{antonacopoulos2009realistic,luong2012logical,soto2019visual} has been studied extensively.
However, form-like documents often have more complex layouts that contain structured objects, such as tables and columns. 
Therefore, form documents have unique challenges compared to natural language documents stemming from their structural characteristics, and have been largely under-explored.

In this work, we study critical information extraction from form documents, which is the fundamental subtask of form document understanding. 
Following the success of sequence modeling in natural language understanding (NLU), a natural approach to tackle this problem is to first serialize the form documents and then apply state-of-the-art sequence models to them. 
For example, \citet{palm2017cloudscan} use Seq2seq~\cite{sutskever2014sequence} with RNN, and ~\citet{hwang2019post} use transformers~\cite{vaswani2017attention}. 
However, interwoven columns, tables, and text blocks make serialization difficult, substantially limiting the performance of a strict serialization approach.

\begin{figure}[t]
    \centering
    \includegraphics[width=0.96\linewidth]{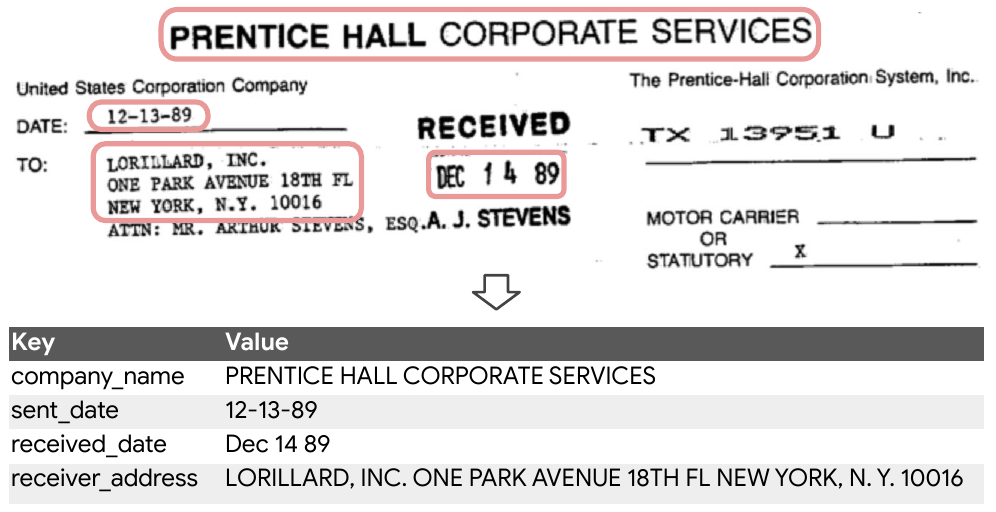}
    \vspace{-2mm}
    \caption{An illustration of the form document information extraction task.}
    \label{fig:ee_task}
    \vspace{-4mm}
\end{figure}

To model the structural information present in documents, \citet{katti2018chargrid,zhao2019cutie,denk2019bertgrid}
treat the documents as 2D image inputs and directly apply convolutional networks on them to preserve the spatial context during learning and inference. 
However, the performance is limited by the resolution of the 2D input grids. 
Another approach is a two-step pipeline ~\cite{hirano2007text} that leverages computer vision algorithms to first infer the layout structures of forms and then perform sequence information extraction. 
The methods are mostly demonstrated on plain text articles or documents~\cite{yang2017learning,soto2019visual} but not on highly entangled form documents~\cite{davis2019deep, zhang2019information}.

\begin{figure*}[t!]
    \centering
    \includegraphics[width=0.9\linewidth]{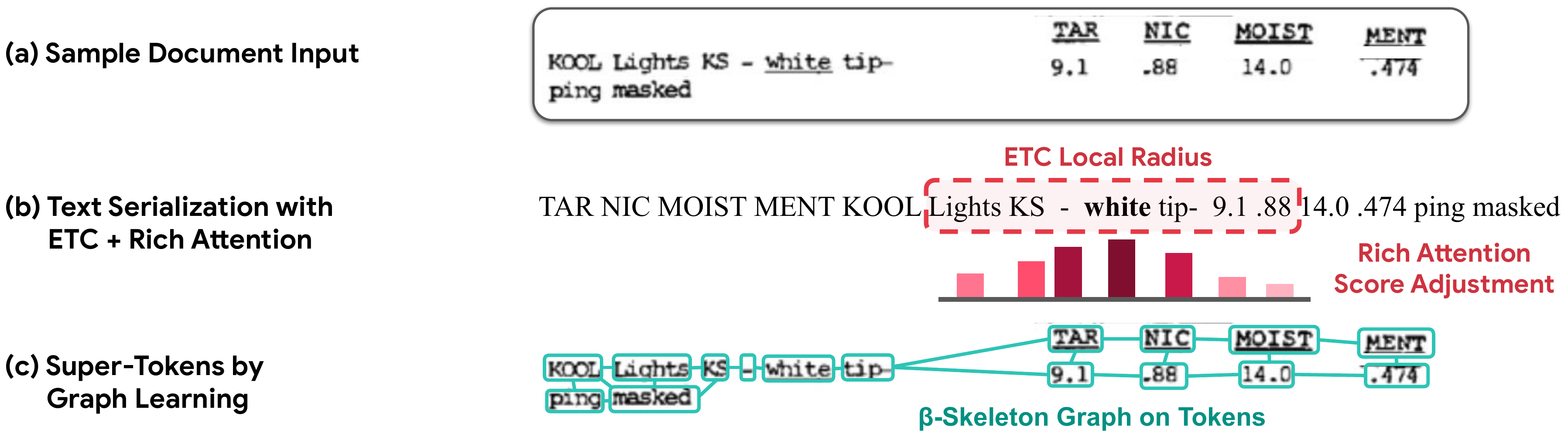}
    \vspace{-1mm}
    \caption{A walk-through example of the proposed Rich Attention and Super-Tokens of FormNet.
    (a) Input document.
    (b) The ETC transformer~\cite{ainslie2020etc} -- the core of our system -- is able to model long inputs by limiting attention to a local radius of serialized tokens. The proposed Rich Attention uses the spatial relationship between tokens to penalize unlikely attention edges.
    In this example, for the word `white', Rich Attention increases the relative weight for spatially nearby tokens such as `KS', `-' and `tip', while decreasing it for others, resulting in spatially aware attention scores.
    \label{fig:serialization}
    (c) Even though they belong to the same entity, \texttt{KOOL} and \texttt{masked} may not be visible to each other from within the local radius of ETC after the left-to-right, top-to-bottom serialization step, which breaks the text group on the left into multiple text segments.
    Our proposed Super-Tokens are generated by executing graph convolutional networks directly on 2D tokens before serialization. The edges of the graph leverage the inductive bias of the $\beta$-skeleton graph, allowing information propagation w.r.t.\ the structural layout of documents before text serialization introduces noise. Note that the $\beta$-skeleton graph connects \texttt{KOOL} and \texttt{masked} in this example.}
    \vspace{-2mm}
\end{figure*}

In this work, we propose FormNet, a structure-aware sequence model to mitigate the suboptimal serialization of forms by bridging the gap between plain sequence models and grid-like convolutional models. 
Specifically, we first design \emph{Rich Attention}, which leverages the spatial relationships between tokens in a form to calculate a more structurally meaningful attention score, and apply it in a recent transformer architecture for long documents~\cite{ainslie2020etc}.
Second, we construct \emph{Super-Tokens} for each word in a form by embedding representations from their neighboring tokens through graph convolutions. 
The graph construction process leverages strong inductive biases about how tokens are related to one another spatially in forms. 
Essentially, given a form document, FormNet builds contextualized \emph{Super-Tokens} before serialization errors can be propagated. 
A transformer model then takes these \emph{Super-Tokens} as input to perform sequential entity tagging and extraction.

In our experiments, FormNet outperforms existing methods while using (1) smaller model sizes and (2) less pre-training data while (3) avoiding the need for vision features. 
In particular, FormNet achieves new best F1 scores on CORD and FUNSD (97.28\% and 84.69\%, respectively) while using a 64\% sized model and 7.1x less pre-training data than the most recent DocFormer~\cite{appalaraju2021docformer}.

\section{Related Work}
Document information extraction was first studied in handcrafted rule-based models~\cite{lebourgeois1992fast,o1993document,ha1995recursive,simon1997fast}. Later 
\citet{marinai2005artificial,shilman2005learning,wei2013evaluation,chiticariu2013rule,schuster2013intellix} use learning-based approaches with engineered features. 
These methods encode low-level raw pixels ~\cite{marinai2005artificial} or assume form templates are known \emph{a priori}~\cite{chiticariu2013rule,schuster2013intellix}, which limits their generalization to documents with specific layout structures. 

\begin{figure*}[t!]
    \centering
    \includegraphics[width=0.98\linewidth]{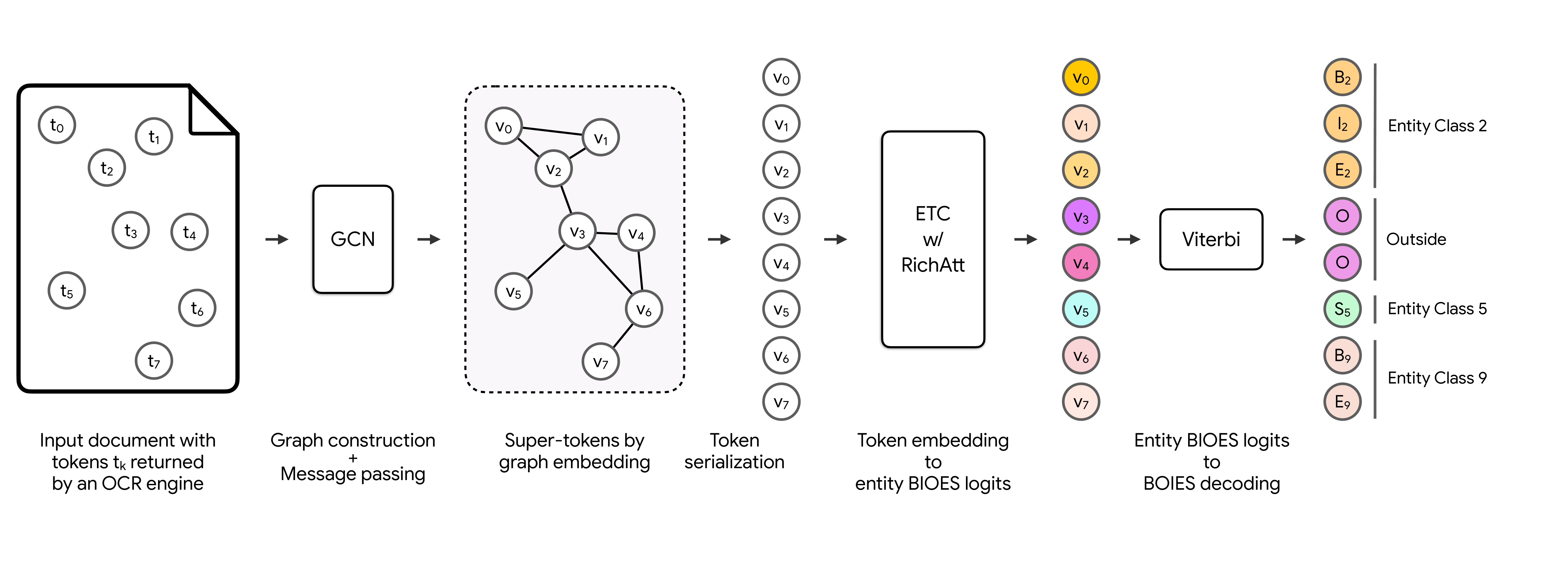}
    \vspace{-1mm}
    \caption{System overview of the proposed FormNet for form document key information extraction. 
    Given a document, we first use the BERT-multilingual vocabulary to tokenize the extracted OCR words. 
    We then feed the tokens and their corresponding 2D coordinates into a Graph Convolutional Network (GCN) for graph construction and message passing. 
    Next, we use ETC~\cite{ainslie2020etc} transformers with the proposed Rich Attention (RichAtt) mechanism to process the GCN-encoded structure-aware Super-Tokens for schema learning. 
    Finally, the Viterbi algorithm is used to decode and obtain the final entity extraction outputs.}
    \label{fig:overview}
    \vspace{-2mm}
\end{figure*}

In addition to models with limited or no learning capabilities, neural models have also been studied. 
\citet{palm2017cloudscan,aggarwal2020form2seq} use an RNN for document information extraction, while \citet{katti2018chargrid,zhao2019cutie,denk2019bertgrid} investigate convolutional models.
There are also self-attention networks (transformers) for document information extraction, motivated by their success in conventional NLU tasks.
\citet{majumder2020representation} extend BERT to representation learning for form documents. \citet{garncarek2020lambert} modified the attention mechanism in RoBERTa~\cite{liu2019roberta}.
\citet{xu2020layoutlm,xu2020layoutlmv2,powalski2021going,appalaraju2021docformer} are multimodal models that combine BERT-like architectures~\cite{devlin2018bert} and advanced computer vision models to extract visual content in images. 
Similarly, SPADE~\cite{hwang2020spatial} is a graph decoder built upon the transformer models for better structure prediction compared to simple BIO tagging.
The proposed FormNet is orthogonal to multimodal transformers and SPADE. Compared with multimodal models, FormNet focuses on modeling relations between words through graph convolutional learning as well as Rich Attention without using any visual modality; 
compared with SPADE, FormNet uses a graph encoder to encode inductive biases in form input. 
A straightforward extension would be to combine FormNet with either layout transformers or SPADE for capturing visual cues or better decoding, which we leave for future work.

Graph learning with sequence models has also been studied. 
On top of the encoded information through graph learning,
\citet{qian2018graphie,liu2019graph,yu2020pick} use RNN and CRF while we study Rich Attention in FormNet for decoding.
\citet{peng2017cross,song2018n} do not study document information extraction.

\section{FormNet for Information Extraction}
\paragraph{Problem Formulation.} 
%Given a form document, we first run the Optical Character Recognition (OCR) engine to identify words and serialize them into sequences\footnote{Different OCR engines implement different heuristics. One common approach is left-to-right top-to-bottom serialization based on 2D positions.} during final inference.
%
Given serialized\footnote{Different Optical Character Recognition (OCR) engines implement different heuristics. One common approach is left-to-right top-to-bottom serialization based on 2D positions.} words of a form document, we formulate the problem as sequential tagging for tokenized words by predicting the corresponding key entity classes for each token.
Specifically, we use the ``BIOES'' scheme -- $\: \{$Begin, Inside, Outside, End, Single$\}$~\cite{ratinov2009design} to mark the spans of entities in token sequences and then apply the Viterbi algorithm.

\vspace{-2mm}
\paragraph{Proposed Approach.} 
By treating the problem as a sequential tagging task after serialization, we can adopt any sequence model.
To handle potentially long documents (e.g. multi-page documents), we select the long-sequence transformer extension ETC~\cite{ainslie2020etc} as our backbone\footnote{One can replace ETC with other long-sequence models, such as \citet{zaheer2020big}.}.

In practice, it is common to see an entity sequence cross multiple spans of a form document, demonstrating the difficulty of recovering from serialization errors.
As illustrated in Figure~\ref{fig:serialization}(b), 
\texttt{9.1} is next to \texttt{tip-},
while \texttt{ping masked} belong to the same entity as \texttt{tip-} are distant from it under the imperfect serialization.
Our remedy is to encode the original 2D structural patterns of forms \emph{in addition} to positions within the serialized sentences.
We propose two novel components to enhance ETC: \emph{Rich Attention} and \emph{Super-Tokens} (Figure~\ref{fig:serialization}). Rich Attention captures not only the semantic relationship but also the spatial distance between every pair of tokens in ETC's attention component. Super-tokens are constructed by graph convolutional networks before being fed into ETC. They model local relationships between pairs of tokens that might not be visible to each other or correctly inferred in an ETC model after suboptimal serialization. 

Figure~\ref{fig:overview} shows the overall system pipeline.
We discuss the details of ETC in Sec.\ \ref{fie_etc}, Rich Attention in Sec.\ \ref{sec:rich}, and Super-Token in Sec.\ \ref{sec:gcn}.

\subsection{Extended Transformer Construction}\label{fie_etc}
Transformers~\cite{vaswani2017attention} have demonstrated state-of-the-art performance on sequence modeling compared with RNNs. 
Extended Transformer Construction~\citep[ETC;][]{ainslie2020etc} further scales transformers to long sequences by replacing standard (quadratic complexity) attention with a sparse global-local attention mechanism. The small number of ``dummy'' {\bf global} tokens attend to all input tokens, but the input tokens attend only {\bf locally} to other input tokens within a specified local radius. An example can be found in Figure~\ref{fig:serialization}(b).  
As a result, space and time complexity are linear in the long input length for a fixed local radius and global input length. Furthermore, ETC allows a specialized implementation for efficient computation under this design. We refer interested readers to \citet{ainslie2020etc} for more details. In this work, we adopt ETC with a single global token as the backbone, as its linear complexity of attention with efficient implementation is critical to long document modeling in practice (e.g. thousands of tokens per document).

A key component in transformers for sequence modeling is the positional encoding~\cite{vaswani2017attention}, which models the positional information of each token in the sequence. Similarly, the original implementation of ETC uses~\citet{shaw2018self} for (relative) positional encoding.
However, token offsets measured based on the error-prone serialization may limit the power of positional encoding. We address this inadequacy by proposing \emph{Rich Attention} as an alternative, discussed in Section \ref{sec:rich}.

\subsection{Rich Attention}

\paragraph{Approach.}
Our new architecture -- inspired by work in dependency parsing \citep{dozat2019arc}, and which we call \emph{Rich Attention} -- avoids the deficiencies of absolute and relative embeddings~\cite{shaw2018self} by avoiding embeddings entirely. Instead, we compute the \emph{order of} and \emph{log distance between} pairs of tokens with respect to the x and y axis on the layout grid, and adjust the pre-softmax attention scores of each pair as a direct function of these values.\footnote{\emph{Order} on the y-axis answers the question ``which token is above/below the other?''} 
At a high level, for each attention head at each layer $\ell$, the model examines each pair of token representations $\mathbf{h}^\ell_i, \mathbf{h}^\ell_j$, whose actual order (using curly Iverson brackets) and log-distance are
\[
    o_{ij} = \{i < j\}\mbox{ and } d_{ij} = \ln(1 + |i - j|).
\]
%\begin{small}
%\begin{align}
%    \label{actual order}o_{ij} &= \{i < j\}\nonumber \\
%    \label{actual distance}d_{ij} &= |i - j|. \nonumber
%\end{align}
%\end{small}

\begin{figure}
    \centering
    \includegraphics[width=0.55\linewidth]{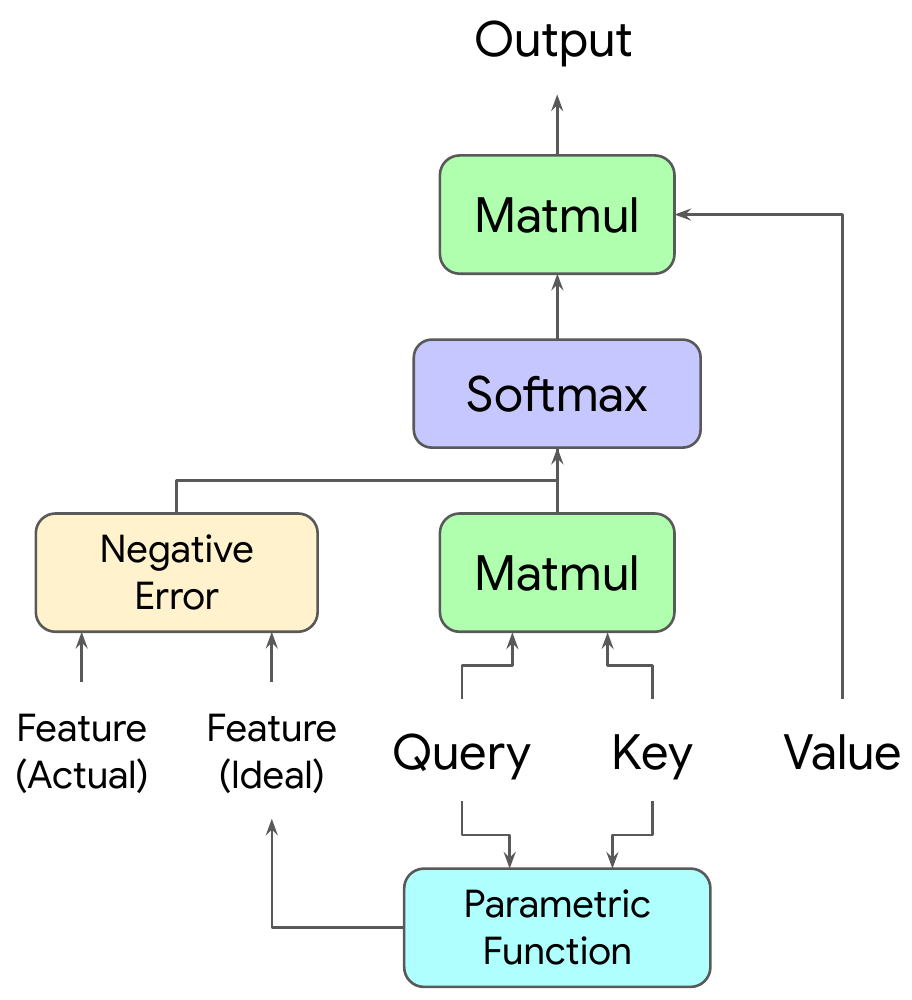}
    \vspace{-2mm}
    \caption{The network uses the Query and Key vectors to consider what value some low-level feature (e.g.\ distance) \emph{should} take if the tokens are related, and penalizes the attention score based on the error.}
    \label{fig:richatt}
    \vspace{-2mm}
\end{figure}

The model then determines the ``ideal'' orders and distances the tokens should have if there is a meaningful relationship between them.

\vspace{-5mm}
\begin{small}
\begin{align}
    \label{ideal order}p_{ij} &= \texttt{Sigmoid}(\text{affine}^{(p)}([\mathbf{h}^\ell_i; \mathbf{h}^\ell_j]))\\
    \label{ideal distance}\mu_{ij} &= \text{affine}^{(\mu)}([\mathbf{h}^\ell_i; \mathbf{h}^\ell_j])
\end{align}
\end{small}

\vspace{-5mm}\noindent
It compare the prediction and groudtruth using sigmoid cross-entropy and $L_2$ losses:\footnote{$\theta$ is a learned \emph{temperature} scalar unique to each head.}

%At a high level, for each attention head, the model examines each pair of tokens $\mathbf{h}_i, \mathbf{h}_j$, determines the ``ideal'' order and distance they should have if there is a meaningful relationship between them (Eqs.\ \ref{ideal order}, \ref{ideal distance}), and then examines their actual order and distance (Eqs.\ \ref{actual order}, \ref{actual distance}).
%
%\vspace{-3mm}
%\begin{small}
%\begin{align}
%    \label{ideal order}p_{ij} &= \texttt{Sigmoid}(\text{affine}^{(p)}([\mathbf{h}_i; \mathbf{h}_j]))\\
%    \label{ideal distance}\mu_{ij} &= \text{affine}^{(\mu)}([\mathbf{h}_i; \mathbf{h}_j])\\
%    \label{actual order}o_{ij} &= \{i < j\}\\
%    \label{actual distance}d_{ij} &= |i - j|
%\end{align}
%\end{small}
%
%It then decreases the attention score based on the error between the ideal and actual values (Eqs.\ \ref{order score}, \ref{distance score}) 
%This is shown in Figure \ref{fig:richatt}.

\vspace{-5mm}
\begin{small}
\begin{align}
    \label{order score}s^{(o)}_{ij} &= o_{ij}\ln(p_{ij}) + (1-o_{ij})\ln(1-p_{ij})\\
    \label{distance score}s^{(d)}_{ij} &= -\frac{\theta^2(d_{ij} - \mu_{ij})^2}{2}
\end{align}
\end{small}

\vspace{-5mm}
\noindent Finally, these are added to the usual attention score 
\begin{small}
\[
    s_{ij} = \mathbf{q}_i^\top\mathbf{k}_j + s_{ij}^{(o)} + s_{ij}^{(d)},
\]
\end{small}
%\vspace{-3mm}
%\begin{small}
%\begin{align*}
%    %\mathbf{q}_i &= \text{affine}^{(q)}(\mathbf{h}_i)\\
%    %\mathbf{k}_j &= \text{affine}^{(k)}(\mathbf{h}_j)\\
%    s &= \mathbf{q}_i^\top\mathbf{k}_j + s^{(o)} + s^{(d)},
%    \vspace{-4mm}
%\end{align*}
%\end{small}%

\label{sec:rich}
\begin{figure*}
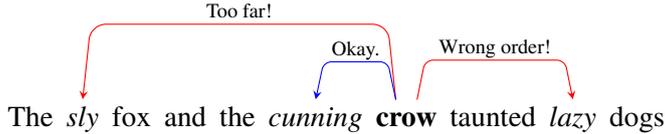

    \centering
    \begin{dependency}[text only label, label style={above}]
        \begin{deptext}
            The \& \textit{sly} \& fox \& and \& the \& \emph{cunning} \& \textbf{crow} \& taunted \& \emph{lazy} \& dogs\\
        \end{deptext}
        \depedge[edge unit distance=1.15ex, color=red]{7}{2}{Too far!}
        \depedge[color=blue]{7}{6}{Okay.}
        \depedge[edge unit distance=1.5ex, color=red]{7}{9}{Wrong order!}
    \end{dependency}
    \vspace{-3mm}
    \caption{A high-level visualization of how rich attention might act on a sentence within a head that composes words with their syntactic modifiers. There are three adjectives that the word \emph{crow} might attend to. However, one of them (\emph{lazy}) is on the wrong side, so its attention edge is penalized. Another (\emph{sly}) is many tokens away, so its attention edge is also penalized. Only one (\emph{cunning}) receives no significant penalties.}
    \label{fig:richatt-sent}
    \vspace{-2mm}
\end{figure*}

\vspace{-5mm}\noindent where $\mathbf{q}_i = \text{affine}^{(q)}(\mathbf{h}_i)$ and $\mathbf{k}_j = \text{affine}^{(k)}(\mathbf{h}_j)$. The rich attention pipeline is shown in Figure \ref{fig:richatt}.\footnote{The affine functions in Eqs.\ (\ref{ideal order}, \ref{ideal distance}) can optionally take the reduced-rank query/key terms $\mathbf{q}_i, \mathbf{k}_j$ as input instead of the layer input $\mathbf{h}^\ell_i, \mathbf{h}^\ell_j$ without sacrificing theoretical motivation. We take this approach for speed.}
By penalizing attention edges for violating these soft order/distance constraints, we essentially build into the model the ability to learn logical implication rules such as ``if $x_i$ is a noun, and $x_j$ is an adjective, and $x_i$ is related (i.e.\ \emph{attends}) to $x_j$, then $x_j$ is to the left of $x_i$''. Note the unidirectionality of this rule -- there could be many unrelated adjectives to the left of $x_i$, so the converse (which this approach \emph{cannot} learn) does not hold in any general sense. This is shown graphically in Figure \ref{fig:richatt-sent}.
%Compared with~\citet{shaw2018self}, we {\bf enrich} the attention mechanisms by properly utilizing additional positional features -- which are easily accessible from OCR engines -- to make the scores layout-structure aware. As such we call it \emph{Rich Attention}. 

{\noindent \bf Justification.}
The approach taken here is not arbitrary. It can be derived algebraically from the probability mass/density functions of the distributions we assume for each feature, and the assumption that a query's attention vector represents a probability distribution. Traditional dot product attention and relative position biases \citep{raffel2020exploring} can likewise be derived from this method, providing incidental justification for the approach. Consider the following, letting $L(X) = \ln(P(X))$ for brevity:

\vspace{-4mm}
\begin{small}
\begin{align}
    P(a_{ij} \mid \mathbf{h}_i,\mathbf{h}_j) &= \frac{P(\mathbf{h}_i, \mathbf{h}_j \mid a_{ij})P(a_{ij})}{\sum\limits_{j'}\left[P(\mathbf{h}_i, \mathbf{h}_{j'} \mid a_{ij'})P(a_{ij'})\right]}\nonumber\\
    &= \frac{\exp\left(L(\mathbf{h}_i,\mathbf{h}_j\mid a_{ij}) + L(a_{ij})\right)}{\sum\limits_{j'}\exp\left(L(\mathbf{h}_i,\mathbf{h}_{j'}\mid a_{ij'}) + L(a_{ij'})\right)}\nonumber\\
    \label{eq:richatt}&= \text{soft}\max_{j'}\left(L(\mathbf{h}_i, \mathbf{h}_{j'} \mid a_{ij'}) + L(a_{ij'})\right)_j
\end{align}
\end{small}

\vspace{-4mm}
\noindent
Here $\mathbf{a}_i$ represents a latent categorical ``attention'' variable. Eq.\ (\ref{eq:richatt}) shows that the softmax function itself can actually be derived from posterior probabilities, by simply applying Bayes' rule and then observing that $x = \exp(\ln(x))$ That is, one need not \emph{define} the posterior as being the softmax of some expression, it \emph{simply is} the softmax of some expression, specifically one that falls out of the assumptions one makes (explicitly or implicitly).
%To derive dot-product attention with relative position biases, we begin by assuming each row of the attention matrix $A$ represents a categorical probability distribution, conditioned on the attender and attendee's hidden vectors $\mathbf{h}_i, \mathbf{h}_j$. With exact equivalence, we can rewrite this expression as the softmax of the log-likelihood and log-prior. Here we let $L(X) = \ln(P(X))$ for brevity and readability.

When we plug the Gaussian probability density function into $L(\mathbf{h}_i, \mathbf{h}_j \mid a_{ij})$, the expression simplifies to dot-product attention (with one additional fancy bias term); we show this in Appendix \ref{appendix: rich attention derivations}. If we assume $L(a_{ij})$ is uniform, then it divides out of the softmax and we can ignore it. If we assume it follows a Bernoulli distribution -- such that $L(a_{ij} = 1; p_{ij}) = \ln(p_{ij})$ -- it becomes equivalent to a learned bias matrix $B$.\footnote{There is an additional constraint that every element of $B$ must be negative; however, because the softmax function is invariant to addition by constants, this is inconsequential.}

Now, if we assume there is another feature $f_{ij}$ that conditions the presence of attention, such as the order or distance of $i$ and $j$, then we can use the same method to derive a parametric expression describing its impact on the attention probability.

\begin{small}
\vspace{-4mm}
\begin{align*}
    &P(a_{ij} \mid f_{ij}, \mathbf{h}_i, \mathbf{h}_j) = \text{soft}\max_{j'}(\\
    &\qquad L(f_{ij'} \mid \mathbf{h}_i, \mathbf{h}_{j'}, a_{ij'}) + L(\mathbf{h}_i, \mathbf{h}_{j'} \mid a_{ij'}) + L(a_{ij'}))_j
\end{align*}
\end{small}
The new term can be expanded by explicating assumptions about the distributions that govern $P(f_{ij} \mid \mathbf{h}_i, \mathbf{h}_j, a_{ij})$ and simplifying the expression that results from substituting their probability functions. If $f_{ij}$ is binary, then this process yields Eq.\ (\ref{order score}), and if $\ln(f_{ij})$ is normally distributed, we reach Eq.\ (\ref{distance score}), as derived in Appendix \ref{appendix: rich attention derivations}.
Given multiple conditionally independent features -- such as the order \emph{and} distance -- their individual scores can be calculated in this way and summed.
Furthermore, relative position biases \citep{raffel2020exploring} can thus be understood in this framework as binary features (e.g.\ $f_{ij} = \{i - j = -2\}$) that are conditionally independent of $\mathbf{h}_i$, $\mathbf{h}_j$ given $a_{ij}$, meaning that $L(f_{ij} \mid \mathbf{h}_i, \mathbf{h}_j, a_{ij}) = L(f_{ij} \mid a_{ij})$.

We call this new attention paradigm \emph{Rich Attention} because it allows the attention mechanism to be \emph{enriched} with an arbitrary set of low-level features. We use it to add order/distance features with respect to the x and y axes of a grid -- but it can also be used in a standard text transformer to encode order/distance/segment information, or it could be used in an image transformer \citep{parmar2018image} to encode relative pixel angle/distance information\footnote{The \emph{von Mises} or wrapped normal distribution would be most appropriate for angular features.}, without resorting to lossy quantization and finite embedding tables.

\subsection{Super-Token by Graph Learning}
\label{sec:gcn}
The key to sparsifying attention mechanisms in ETC~\cite{ainslie2020etc} for long sequence modeling is to have every token only attend to tokens that are within a pre-specified local radius in the serialized sequence. The main drawback to ETC in form understanding is that imperfect serialization sometimes results in entities being serialized too far apart from each other to attend in the local-local attention component (i.e. outside the local radius). A naive solution is to increase the local radius in ETC. However, it sacrifices the efficiency for modeling long sequences. Also, the self-attention may not be able to fully identify relevant tokens when there are many distractors~\citep[Figure~\ref{fig:visualization1};][]{serrano2019attention}.

To alleviate the issue, we construct a graph to connect nearby tokens in a form document. We design the edges of the graph based on strong inductive biases so that they have higher probabilities of belonging to the same entity type (Figure~\ref{fig:serialization}(c) and 
\ref{beta_skeleton_example}). 
Then, for each token, we obtain its \emph{Super-Token} embedding by applying graph convolutions along these edges to aggregate semantically meaningful information from its neighboring tokens. 
We use these super-tokens as input to the Rich Attention ETC for sequential tagging.
This means that even though an entity may have been broken up into multiple segments due to poor serialization, the super-tokens learned by the graph convolutional network will have recovered much of the context of the entity phrase.
We next introduce graph construction and the learning algorithm.

\vspace{-2mm}
\paragraph{Node Definition.}
Given a document with $N$ tokens denoted by $T = \{t_1, t_2, \ldots t_N\}$, we let $t_k$ refer to the $k$-th token in a text sequence returned by the OCR engine. The OCR engine generates the bounding box sizes and locations for all tokens, as well as the text within each box. We define node input representation for all tokens $T$ as vertices $V = \{\mathbf{v}_1, \mathbf{v}_2, \ldots \mathbf{v}_N\}$, where $\mathbf{v}_k$ concatenates attributes available for $t_k$. In our design, we use three common input modalities: (a) one-hot word embeddings, (b) spatial embeddings from the normalized Cartesian coordinate values of the four corners and height and width of a token bounding box~\cite{qian2018graphie,davis2019deep,liu2019graph}. The benefit of representing tokens in this way is that one can add more attributes to a vertex by simple concatenation without changing the macro graph architecture.

\vspace{-2mm}
\paragraph{Edge Definition.}
While the vertices $V$ represent tokens in a document, the edges characterize the relationship between all pairs of vertices. Precisely, we define directed edge embeddings for a set of edges $E$, where each edge $\mathbf{e}_{kl}$ connects two vertices $\mathbf{v}_k$ and $\mathbf{v}_l$, concatenating quantitative edge attributes. In our design, the edge embedding is composed of the relative distance between the centers, top left corners, and bottom right corners of the token bounding boxes. The embedding also contains the shortest distances between the bounding boxes along the horizontal and vertical axis. Finally, we include the height and width aspect ratio of $\mathbf{v}_k$, $\mathbf{v}_l$, and the bounding box that covers both of them.

\vspace{-2mm}
\paragraph{Graph construction.}
After contructing edge embeddings, we need discrete graphs to define connectivities.
One approach would be to create k-Nearest-Neighbors graphs~\cite{zhang2020deeprelational} -- but 
these may contain isolated components, which is not ideal for information propagation. Instead, we construct graphs using the $\beta$-skeleton algorithm~\cite{kirkpatrick1985framework} with $\beta=1$, which is found useful for document understanding in~\citet{wang2021general,lee2021rope}. It essentially creates a ``ball-of-sight'' graph with a linearly-bounded number of edges while also guaranteeing global connectivity as shown in Figure~\ref{beta_skeleton_example}.
More examples of constructed $\beta$-skeleton graphs can be found in Figure~\ref{fig:beta_skeleton} in the Appendix.

\vspace{-2mm}
\paragraph{Message passing.}
Graph message-passing is the key to propagating representations along the edges defined by the inductive bias, $\beta$-skeleton, that are free from the left-to-right top-to-bottom form document serialization. In our design, we perform graph convolutions~\citep[GCN;][]{gilmer2017neural} on concatenated features from pairs of neighboring nodes and edges connecting them. Hence the graph embedding is directly learned from back-propagation in irregular patterns of tokens in documents.

\begin{figure}
    \centering
    \includegraphics[width=0.9\linewidth]{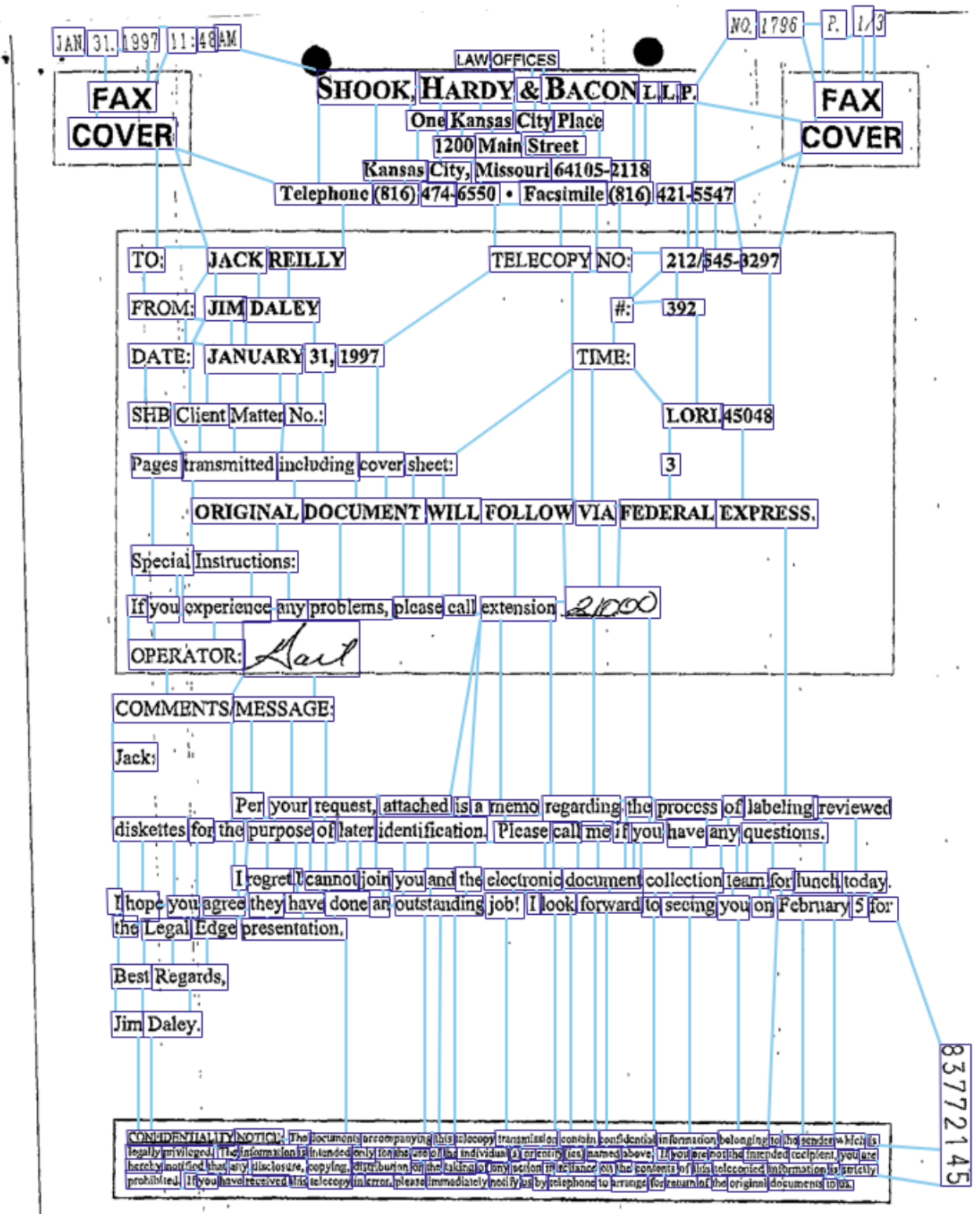}
    \vspace{-3mm}
    \caption{An illustration of the word-level $\beta$-skeleton graph of a FUNSD document, which is a  sparse but connected graph.}
    \label{beta_skeleton_example}
    \vspace{-3mm}
\end{figure}

\begin{table*}[!ht]
\setlength{\tabcolsep}{6pt} % Default value: 6pt
\centering
\resizebox{1\textwidth}{!}{
\footnotesize{
\begin{tabular}{llccccccc}
\toprule
\textbf{Dataset} & \textbf{Method} & \textbf{P} & \textbf{R} & \textbf{F1} & \textbf{Image} & \textbf{\#Params} &  \textbf{Pre-training Size}   \\
\toprule
% CORD & \textsc{w/o \hspace{0.5mm} image modality} \\
CORD & SPADE~\cite{hwang2020spatial} & - & - & 91.5 & & 110M & BERT-multilingual \\
& UniLMv2~\cite{bao2020unilmv2} & 91.23 & 92.89 & 92.05 & & 355M & 160GB\\
& LayoutLMv1~\cite{xu2020layoutlmv2} & 94.32 & 95.54 & 94.93 & & 343M & 11M\\
& DocFormer~\cite{appalaraju2021docformer} & 96.46 & 96.14 & 96.30 & & 502M & 5M\\
\cmidrule{2-8}
& LayoutLMv2~\cite{xu2020layoutlmv2} & 95.65 & 96.37 & 96.01 & \checkmark & 426M & 11M\\
& TILT~\cite{powalski2021going} & - & - & 96.33 & \checkmark & 780M & 1.1M\\
& DocFormer~\cite{appalaraju2021docformer} & 97.25 & 96.74 & 96.99 & \checkmark & 536M & 5M\\
\cmidrule{2-8}
& FormNet (ours) & 98.02 & 96.55 & \textbf{97.28} & & 345M & 0.7M (9GB) \\
\midrule
FUNSD & SPADE~\cite{hwang2020spatial} & - & - & 70.5 & & 110M & BERT-multilingual \\
& UniLMv2~\cite{bao2020unilmv2} & 67.80 & 73.91 & 70.72 & & 355M & 160GB\\
& LayoutLMv1~\cite{xu2020layoutlm} & 75.36 & 80.61 & 77.89 & & 343M & 11M \\
& DocFormer~\cite{appalaraju2021docformer} & 81.33 & 85.44 & 83.33 & & 502M & 5M \\
\cmidrule{2-8}
& LayoutLMv1~\cite{xu2020layoutlm} & 76.77 & 81.95 & 79.27 & \checkmark & 160M & 11M \\
& LayoutLMv2~\cite{xu2020layoutlmv2} & 83.24 & 85.19 & 84.20 & \checkmark & 426M & 11M \\
& DocFormer~\cite{appalaraju2021docformer} & 82.29 & 86.94 & 84.55 & \checkmark & 536M & 5M \\
\cmidrule{2-8}
& FormNet (ours)  & 85.21 & 84.18 & \textbf{84.69} & & 217M &  0.7M (9GB)\\
\midrule
Payment & NeuralScoring~\cite{majumder2020representation} & - & - & 87.80 & & - & 0\\
        & FormNet (ours)  & 92.70 & 91.69 & \textbf{92.19} & & 217M & 0\\
\bottomrule
\end{tabular}
}
}
\vspace{-2mm}
\caption{\label{comparison} Entity-level precision, recall, and F1 score comparisons on three standard benchmarks. The proposed FormNet establishes new state-of-the-art results on all three datasets.
On FUNSD and CORD, FormNet significantly outperforms the most recent DocFormer~\cite{appalaraju2021docformer} while using a 64\% sized model and 7.1x less pre-training data. For detailed FormNet family performance please see Table~\ref{formnet_family_ablation}.}
\vspace{-2mm}
\end{table*}

\section{Evaluation}
\label{sec:exp}
We evaluate how the two proposed structural encoding components, Rich Attention and Super-Tokens, impact the overall performance of form-like document key information extraction. We perform extensive experiments on three standard benchmarks\footnote{We note that SROIE~\cite{huang2019icdar2019} and Kleister-NDA~\cite{gralinski2020kleister} are designed for key-value pair extraction instead of direct entity extraction. We leave the work of modifying FormNet for key-value pair extraction in the future.} and compare the proposed method with recent competing approaches.

\subsection{Datasets}
\paragraph{CORD.}
We evaluate on CORD~\cite{park2019cord}, which stands for the Consolidated Receipt Dataset for post-OCR parsing.
%is a large-scale public dataset that contains over 11,000 Indonesian receipts from shops and restaurants through crowd-sourcing. 
The annotations are provided in 30 fine-grained semantic entities such as store name, menu price, table number, discount, etc. 
%For the entity extraction task, 
We use the standard evaluation set that has 800 training, 100 validation, and 100 test samples.

\vspace{-2mm}
\paragraph{FUNSD.}
FUNSD~\cite{jaume2019} is a public dataset for form understanding in noisy scanned documents. It is a subset of the Truth Tobacco Industry Document (TTID)\footnote{http://industrydocuments.ucsf.edu/tobacco}.
%containing a collection of research, marketing, and advertising documents that vary widely in their structure and appearance. 
The dataset consists of 199 annotated forms with 9,707 entities and 31,485 word-level annotations for 4 entity types: header, question, answer, and other.
%For the entity extraction task, 
We use the official 75-25 split for the training and test sets.

\vspace{-2mm}
\paragraph{Payment.}
We use the large-scale payment data \cite{majumder2020representation} that consists of around 10K documents and 7 semantic entity labels from human annotators.
The corpus comes from different vendors with different layout templates. 
%Each entity ground truth is described by an entity type and a list of words generated by an OCR engine.
%In all of our entity extraction experiments, 
We follow the same evaluation protocol and dataset splits used in~\citet{majumder2020representation}.

\begin{figure}[t!]
    \centering
    \includegraphics[width=0.98\linewidth]{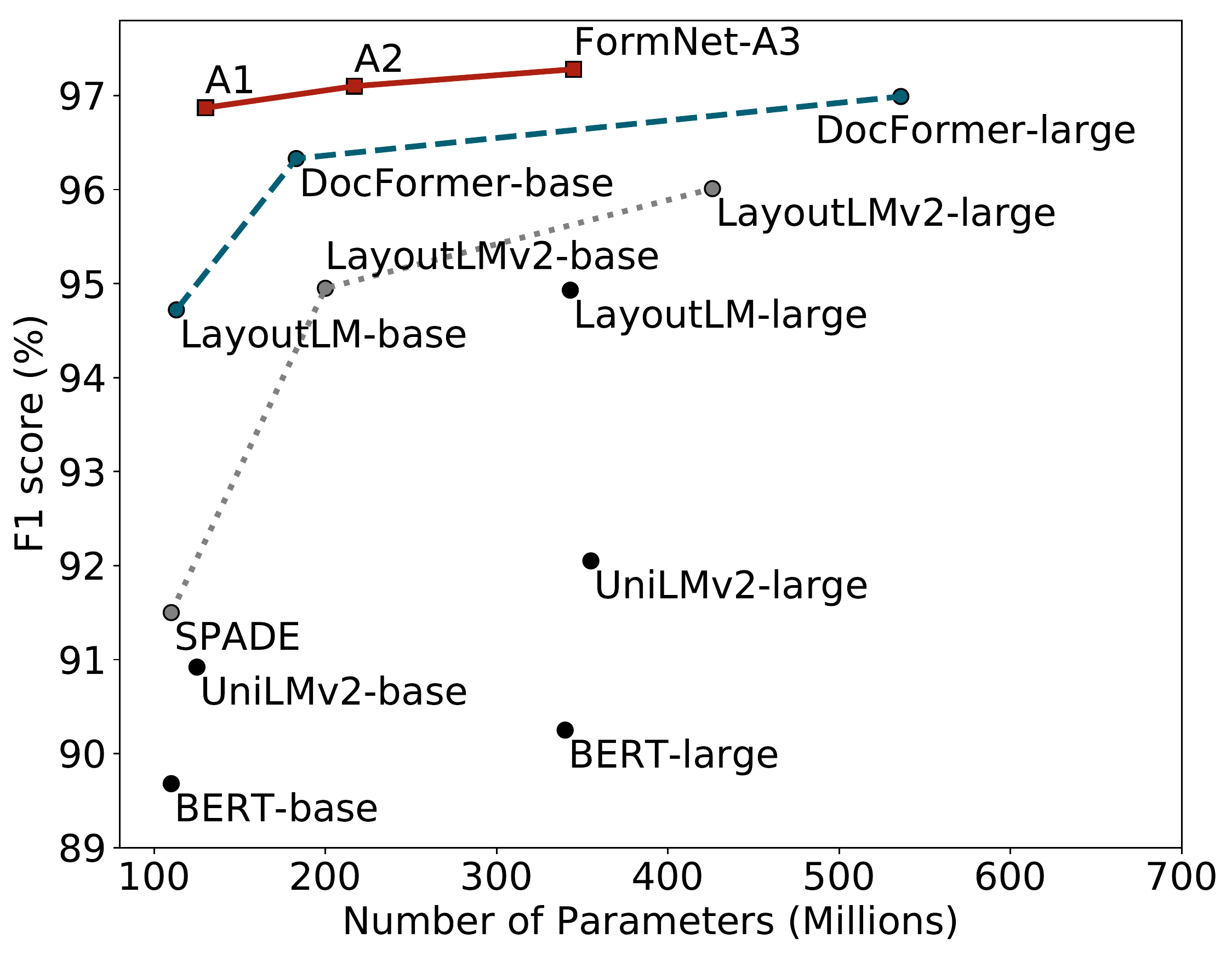}
    \vspace{-2mm}
    \caption{\textbf{Model Size vs. Entity Extraction F1 Score} on CORD benchmark. The proposed FormNets significantly outperform other recent approaches -- FormNet-A2 achieves higher F1 score (97.10\%) while using a 2.5x smaller model and 7.1x less pre-training data than DocFormer~\citep[96.99\%;][]{appalaraju2021docformer}. FormNet-A3 obtains the highest 97.28\% F1 score.}
    \label{fig:cord_params}
    \vspace{-2mm}
\end{figure}

\subsection{Experimental Setup}
\label{sec:setup}
Given a document, we first use the BERT-multilingual vocabulary to tokenize the extracted OCR words.
%
%We then feed the tokens and their corresponding 2D spatial coordinates into a GCN. 
Super-tokens are then generated by direct graph embedding on these 2D tokens.
Next, we use ETC transformer layers to continue to process the super-tokens based on the serialization provided by the corresponding datasets. 
Please see Appendix~\ref{implementation_appendix} for implementation details.

We use 12-layer GCN and 12-layer ETC in FormNets and scale up the FormNet family with different numbers of hidden units and attention heads to obtain FormNet-A1 (512 hidden units and 8 attention heads), A2 (768 hidden units and 12 attention heads), and A3 (1024 hidden units and 16 attention heads). Ablations on the FormNets can be found in Figure~\ref{fig:cord_params} and~\ref{pretrain}, and Table~\ref{formnet_family_ablation} in Appendix.

\vspace{-2mm}
\paragraph{MLM Pre-training.}
%Following~\citet{xu2020layoutlm,xu2020layoutlmv2}, 
Following~\citet{appalaraju2021docformer}, we collect around 700k unlabeled form documents for unsupervised pre-training. 
%instead of pre-training on the standard NLU unlabeled corpus~\cite{devlin2018bert}.
We adopt the Masked Language Model (MLM) objective~\cite{taylor1953cloze, devlin2018bert} to pre-train the networks. 
This forces the networks to reconstruct randomly masked tokens in a document to learn the underlying semantics of language from the pre-training corpus. We train the models from scratch using Adam optimizer with batch size of 512. The learning rate is set to 0.0002 with a warm-up proportion of 0.01.

\vspace{-2mm}
\paragraph{Fine-tuning.}
We fine-tune all models in the experiments using Adam optimizer with batch size of 8. The learning rate is set to 0.0001 without warm-up. We use cross-entropy loss for the multi-class BIOES tagging tasks. The fine-tuning is conducted on Tesla V100 GPUs for approximately 10 hours on the largest corpus. Note that we only apply the MLM pre-training for the experiments on CORD and FUNSD as in~\citet{xu2020layoutlm,xu2020layoutlmv2}.
For the experiments on Payment, we follow~\citet{majumder2020representation} to directly train all networks from scratch without pre-training.

\subsection{Results}
\paragraph{Benchmark Comparison.}
Table~\ref{comparison} lists the results that are based on the same evaluation protocal\footnote{Micro-F1 for CORD and FUNSD by following the implementation in~\citet{xu2020layoutlmv2}; macro-F1 for Payment~\cite{majumder2020representation}.}. The proposed FormNet achieves the new best F1 scores on CORD, FUNSD, and Payment benchmarks.
%with 8.41 points, 1.08 points, and 4.39 points improvement over the most recent competing methods, respectively. 
Figure~\ref{fig:cord_params} shows model size vs.\ F1 score for all recent approaches.
%
%Models pretrained on the standard corpus crawled from internet instead of form documents (e.g.\ BERT) are clearly suboptimal.
On CORD and FUNSD, FormNet-A2 (Table~\ref{formnet_family_ablation} in Appendix)  outperforms the most recent DocFormer~\cite{appalaraju2021docformer} while using a 2.5x smaller model and 7.1x less unlabeled pre-training documents.
On the larger CORD, FormNet-A3 continues to improve the performance to the new best 97.28\% F1.
In addition, we observe no difficulty training the FormNet from scratch on the Payment dataset. These demonstrate the  parameter efficiency and the training sample efficiency of the proposed FormNet.

\begin{figure}
    \centering
    \includegraphics[width=0.98\linewidth]{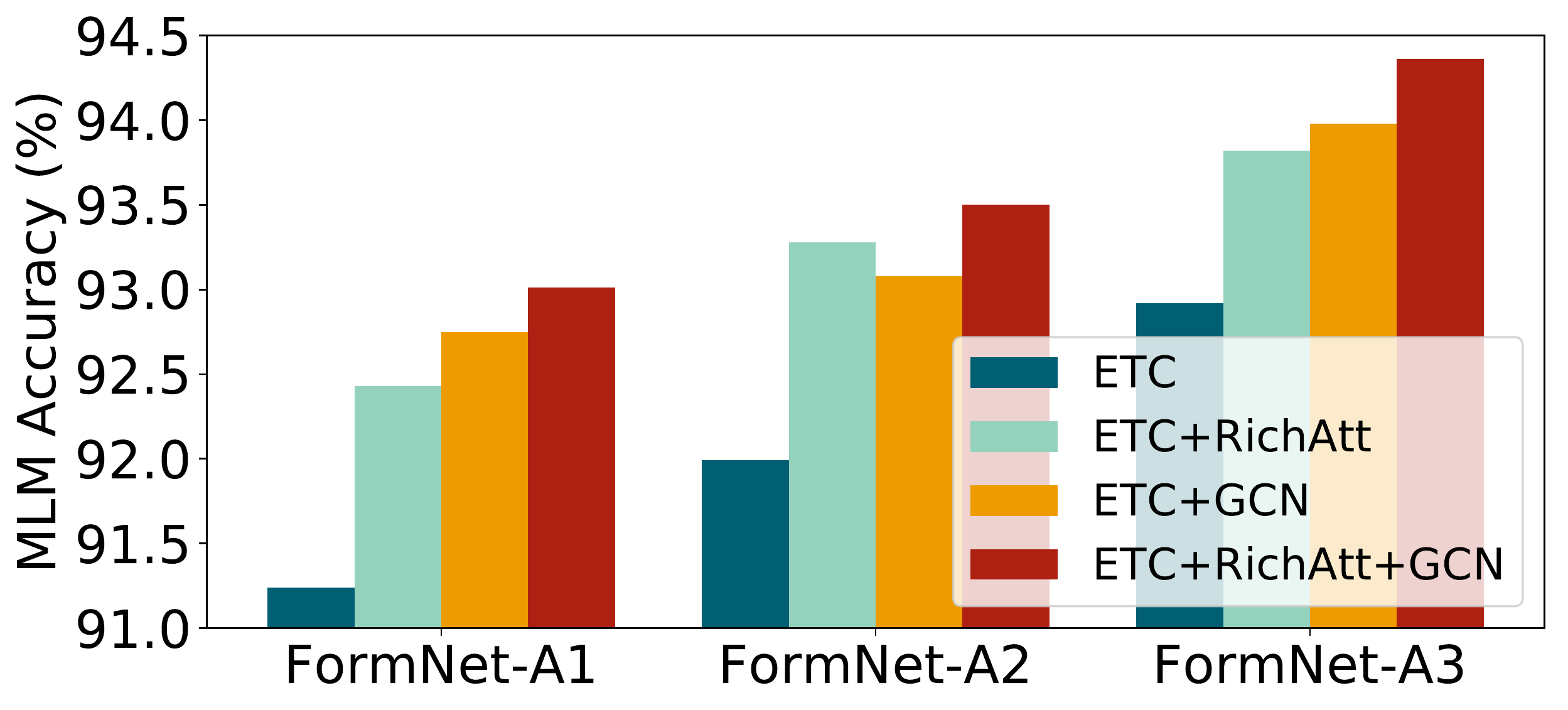}
    \vspace{-2mm}
    \caption{Performance of the MLM pre-training. Both the proposed Rich Attention (RichAtt) and Super-Token by Graph Convolutional Network (GCN) components improve upon ETC~\cite{ainslie2020etc} baseline by a large margin, showing the effectiveness of their structural encoding capability on large-scale form documents.}
    \label{pretrain}
    \vspace{0mm}
\end{figure}

\begin{table}[!t]
\setlength{\tabcolsep}{7pt} % Default value: 6pt
\centering
\footnotesize
\resizebox{0.43\textwidth}{!}{
\footnotesize{
\begin{tabular}{lccccccccc}
\toprule
 & \textbf{RichAtt} & \textbf{GCN} & \textbf{P} & \textbf{R} & \textbf{F1}   \\
\midrule
\multirow{4}{*}{\rotatebox{90}{CORD}}  &  &   & 91.40 & 91.75 & 91.57  \\
                    & \checkmark &            & 97.28 & 95.19 & 96.03  \\
                    &            & \checkmark & 96.50 & 95.13 & 95.81  \\
                    & \checkmark & \checkmark & 97.50 & 96.25 & \textbf{96.87}  \\      
\midrule
\multirow{4}{*}{\rotatebox{90}{FUNSD}}  &  &  & 69.24 & 62.86 & 65.90  \\
                    & \checkmark &            & 82.16 & 82.28 & 82.22  \\
                    &            & \checkmark & 78.83 & 79.93 & 79.37  \\
                    & \checkmark & \checkmark & 84.17 & 84.88 & \textbf{84.53}  \\                    
\bottomrule
\end{tabular}
}
}
\vspace{-2mm}
\caption{\label{fine-tune} Ablation of the proposed Rich Attention (RichAtt) and Super-Token by Graph Convolutional Network (GCN) in entity-level precision, recall, and F1 score on CORD and FUNSD benchmarks using FormNet-A1. Both RichAtt and GCN significantly improve upon ETC~\cite{ainslie2020etc} baseline by a large margin.}
\vspace{-1mm}
\end{table}

\vspace{-2mm}
\paragraph{Effect of Structural Encoding in Pre-training.}
We study the importance of the proposed Rich Attention and Super-Token by GCN on the large-scale MLM pre-training task across three FormNets as summarized in Figure~\ref{pretrain}. Both Rich Attention and GCN components improve upon the ETC~\cite{ainslie2020etc} baseline on reconstructing the masked tokens by a large margin, showing the effectiveness of their structural encoding capability on form documents. 
The best performance is obtained by incorporating both.

\begin{figure*}
    \centering
    \includegraphics[width=0.9\linewidth]{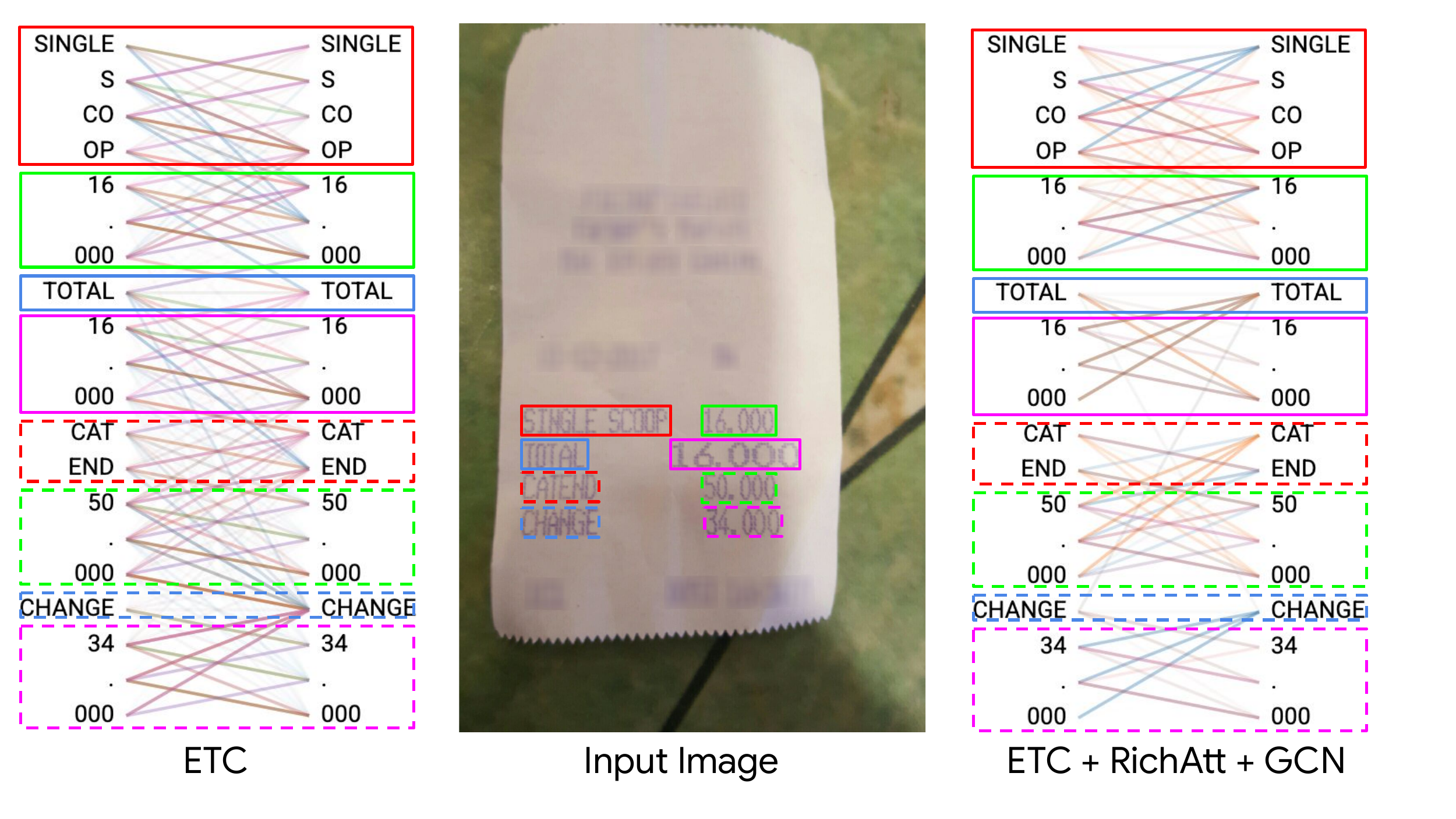}
    \vspace{-2mm}
    \caption{The attention scores for ETC and ETC+RichAtt+GCN models. Unlike the ETC model, the ETC+RichAtt+GCN model makes tokens attend to other tokens within the same visual blocks, along with tokens aligned horizontally, thus strongly leveraging structural cues.}
    \label{fig:visualization1}
    \vspace{0mm}
\end{figure*}

\begin{figure*}[ht]
    \centering
    \includegraphics[width=0.95\linewidth]{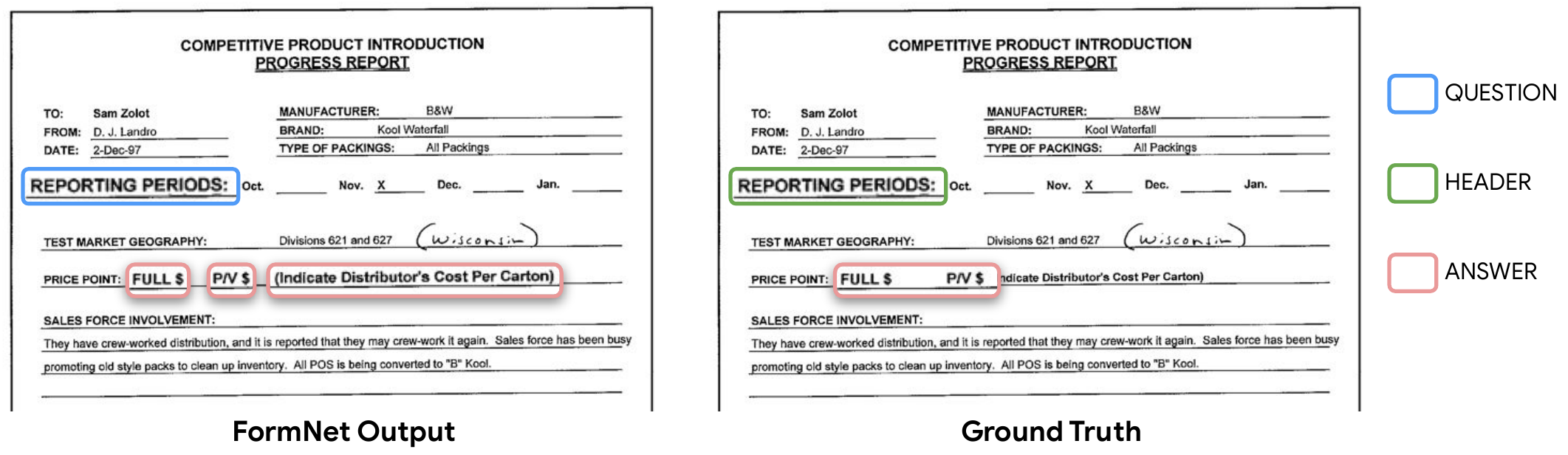}
    \vspace{-3mm}
    \caption{The ambiguous cases where the FormNet predictions do not match the human-annotated ground truth. In this visualization we only showcase mismatched entities.}
    \label{fig:ambiguous}
    \vspace{0mm}
\end{figure*}

\vspace{-2mm}
\paragraph{Effect of Structural Encoding in Fine-tuning.}
We ablate the effect of the proposed Rich Attention and Super-Tokens by GCN on the fine-tuning tasks and measure their entity-level precision, recall, and F1 scores. In Table~\ref{fine-tune}, we see that both Rich Attention and GCN improve upon the ETC~\cite{ainslie2020etc} baseline on all benchmarks. 
In particular, Rich Attention brings 4.46 points and GCN brings 4.24 points F1 score improvement over the ETC baseline on CORD. We also see a total of 5.3 points increase over the baseline when using both components, showing their orthogonal effectiveness of encoding structural patterns. More ablation can be found in Section~\ref{gcn_importance_appendix} and Table~\ref{fine-tune2} in Appendix.

\subsection{Visualization}
Using BertViz~\citep{vig-2019-multiscale}, we visualize the local-to-local attention scores for specific examples of the CORD dataset for the ETC baseline and the ETC+RichAtt+GCN (FormNet) models. Qualitatively in Figure \ref{fig:visualization1}, we notice that the tokens attend primarily to other tokens within the same visual block for ETC+RichAtt+GCN. Moreover for that model, specific attention heads are attending to tokens aligned horizontally, which is a strong signal of meaning for form documents. No clear attention pattern emerges for the ETC model, suggesting the Rich Attention and Super-Token by GCN enable the model to learn the structural cues and leverage layout information effectively. More visualization examples are given in the Appendix \ref{visualization_appendix}. We also show sample model outputs in Figure~\ref{fig:ambiguous}.

\section{Conclusion}
We present a novel model architecture for key entity extraction for forms, FormNet. We show that the proposed Rich Attention and Super-Token components help the ETC transformer to excel at form understanding in spite of noisy serialization, as evidenced quantitatively by its state-of-the-art performance on three benchmarks and qualitatively by its more sensible attention patterns. In the future, we would like to explore multi-modality input such as images.

\clearpage

\bibliography{anthology,custom}
\bibliographystyle{acl_natbib}

\clearpage

\appendix

\section{Implementation Details}
\label{implementation_appendix}
The proposed FormNet consists of a GCN encoder to generate structure-aware super-tokens followed by a ETC transformer decoder equipped with Rich Attention for key information extraction. Each GCN layer is a 2-layer multi-layer Perceptron (MLP) with the same number of hidden units as the ETC transformer. The maximum number of neighbors is set to 8 so the graph convolution computation grows linearly w.r.t. the number of vertices.
An 1-head attention aggregation function is used after each message passing. We also adopt skip-connection and layer-normalization after each GCN calculation. The ETC transformer takes super-tokens as input. The maximum sequence length is set to 1024. We follow~\citet{ainslie2020etc} for other hyper-parameter settings.

\section{Impact of Super-Tokens by Graph Convolutional Networks}
\label{gcn_importance_appendix}
In this experiment, we investigate whether simply increasing the network capacity of the ETC transformer with Rich Attention (RichAtt) can surpass the performance of the FormNet (ETC+RichAtt+GCN). Here ETC-heavy uses 768 hidden units instead of 512 in ETC-standard for both local and global tokens.

Table~\ref{gcn_ablation} shows that this is not the case. Simply increasing the network capacity of the ETC transformer from 104M parameters to 187M parameters only improves the performance by 0.7\% on FUNSD. On the contrary, the proposed Super-Tokens by GCN continues to improve the standard ETC + RichAtt and outperforms ETC-heavy + RichAtt by a large margin.
This evidence suggests that GCN captures the structural information from form documents effectively, which is challenging for ETC due to the limited local radius and multiple text segment issue from imperfect text serialization.
These encourage the design of FormNet to balance between efficiency of modeling long documents (ETC) and effectiveness of modeling structural information (GCN).

\begin{table}[ht]
\setlength{\tabcolsep}{7pt} % Default value: 6pt
\centering
\resizebox{0.49\textwidth}{!}{
\footnotesize{
\begin{tabular}{lccccc}
\toprule
\textbf{Model}  & \textbf{F1} & \textbf{\#Params}   \\
\toprule
 ETC-standard + RichAtt  & 82.22 & 104M \\
 ETC-heavy + RichAtt  & 82.92 & 187M \\
 \midrule
 ETC-standard + RichAtt + GCN & 84.53 & 131M \\
\bottomrule
\end{tabular}
}
}
\vspace{-2mm}
\caption{\label{gcn_ablation} The impact of Super-Tokens by Graph Convolutional Networks (GCNs) compared to heavier ETC transformers. The proposed FormNet (ETC-standard + RichAtt + GCN) significantly outperforms the ETC-heavy + RichAtt counterparts while using much less number of parameters, showing the effectiveness of the structural modeling capability of GCN.
}
\end{table}

\section{Rich Attention Derivations}
\label{appendix: rich attention derivations}
Here we lay out more explicitly the assumptions and steps needed to derive Rich Attention.
First, we assume that there is a latent categorical \emph{attention} feature $a_{ij} \in \{0, 1\}$ that indicates the presence or absence of some unique relevant relationship between tokens $i$ and $j$. In the context of transformers, when $a_{ij} = 1$ (abbreviated simply $a_{ij}$), the ``value'' hidden state $\mathbf{v}_j$ gets combined with token $i$'s context representation and propagated up the network.
\begin{align*}
    \mathbf{c}_i &= \sum\limits_{j}\left[a_{ij}\mathbf{v}_j\right]
\end{align*}
However, since categorical variables are discrete (therefore undifferentiable), we use the (differentiable) \emph{probability} of $a_{ij}$ to compute the \emph{expected} value state instead.
\begin{align*}
    E[\mathbf{c}_i] &= \sum\limits_{j}\left[P(a_{ij} \mid \mathbf{h}_i, \mathbf{h}_j, \ldots)\mathbf{v}_j\right]
\end{align*}
The expressions for $P(a_{ij} \mid \mathbf{h}_i,\mathbf{h}_j)$ and $P(a_{ij} \mid f_{ij}, \mathbf{h}_i, \mathbf{h}_j)$ derived in Section \ref{sec:rich} are repeated below, again letting $L(X) = \ln(P(X))$ for readability.

\begin{small}
\begin{align*}
    &P(a_{ij} \mid \mathbf{h}_i,\mathbf{h}_j) = \text{soft}\max_j(\\
        &\qquad L(\mathbf{h}_i, \mathbf{h}_j \mid a_{ij}) + L(a_{ij}))_i\\
    &P(a_{ij} \mid f_{ij}, \mathbf{h}_i, \mathbf{h}_j) = \text{soft}\max_j(\\
        &\qquad L(f_{ij} \mid \mathbf{h}_i, \mathbf{h}_j, a_{ij}) + L(\mathbf{h}_i, \mathbf{h}_j \mid a_{ij}) + L(a_{ij}))
\end{align*}
\end{small}
Note that here and in future derivations, we only care about the case when $a_{ij} = 1$, meaning the value of $a_{ij}$ is constant and can be effectively ignored. Theorem \ref{theorem: h given a} shows how to derive dot-product attention, Theorem \ref{theorem: binary f} solves for a binary-valued feature, and Theorem \ref{theorem: exp f} solves for a real-valued feature on an exponential scale; we leave the derivation for other feature types and probability distributions as a fun exercise for the reader.

\section{Examples of $\beta$-skeleton Graphs}
Figure~\ref{fig:beta_skeleton} shows a constructed $\beta$-skeleton graph on the public FUNSD dataset. By using the ``ball-of-sight'' strategy, $\beta$-skeleton graph offers high connectivity between word vertices for necessary message passing while being much sparser than fully-connected graphs for efficient forward and backward computations.

\section{Additional Attention Visualization}\label{visualization_appendix}
Figure \ref{fig:visualization2} shows additional attention visualization. The ETC+RichAtt+GCN model has very interpretable attention scores due to its ability to leverage spatial cues. As a result, the model strongly attends to tokens in the same visual blocks, or that have horizontal alignment. Specific heads also have specific roles: the pink head attends to the token on the right (reading order) within a block and captures intra-block semantics. The blue head attends to the previous horizontally-aligned block (in Figure \ref{fig:visualization2}, the tokens "To", "fal", "1560" and "00" all attend to the token "Sub") and captures inter-block semantics.

\begin{table}[h]
\setlength{\tabcolsep}{7pt} % Default value: 6pt
\vspace{5mm}
\centering
\footnotesize
\resizebox{0.48\textwidth}{!}{
\footnotesize{
\begin{tabular}{lcccccccc}
\toprule
\textbf{Dataset} & \textbf{\#Samples} & \textbf{FormNet} & \textbf{P} & \textbf{R} & \textbf{F1}   \\
\midrule
FUNSD & 199  & A1 & 84.17 & 84.88 & 84.53  \\
      &      & A2 & 85.21 & 84.18 & 84.69  \\
\midrule
CORD & 1,000  & A1 & 97.50 & 96.25 & 96.87  \\
     &       & A2 & 97.51 & 96.70 & 97.10  \\
     &       & A3 & 98.02 & 96.55 & 97.28  \\
\bottomrule
\end{tabular}
}
}
\vspace{-2mm}
\caption{\label{formnet_family_ablation} Scaling the FormNet family on CORD and FUNSD benchmarks. FormNet-A2 outperforms the most recent DocFormer~\cite{appalaraju2021docformer} on both datasets while being a 2.5x smaller model. On the larger CORD dataset, FormNet-A3 continues to boost the performance to the new best 97.28\% F1.}
\end{table}

\begin{table}[h]
\setlength{\tabcolsep}{7pt} % Default value: 6pt
\centering
\footnotesize
\resizebox{0.43\textwidth}{!}{
\footnotesize{
\begin{tabular}{lccccccccc}
\toprule
 & \textbf{RichAtt} & \textbf{GCN} & \textbf{P} & \textbf{R} & \textbf{F1}   \\
\midrule
\multirow{4}{*}{\rotatebox{90}{Payment}}  &  &   & 83.91 & 83.27 & 83.58  \\
                    & \checkmark &            & 92.10 & 91.48 & 91.79  \\
                    &            & \checkmark & 87.79 & 84.47 & 86.10  \\
                    & \checkmark & \checkmark & 92.70 & 91.69 & \textbf{92.19}  \\                          
\bottomrule
\end{tabular}
}
}
\vspace{-2mm}
\caption{\label{fine-tune2} Ablation of the proposed Rich Attention (RichAtt) and Super-Token by Graph Convolutional Network (GCN) in entity-level precision, recall, and F1 score on the Payment benchmark using FormNet-A2. Both RichAtt and GCN significantly improve upon ETC~\cite{ainslie2020etc} baseline by a large margin.}
\end{table}

% \begin{figure}[h]
%     \vspace{5mm}
%     \centering
%     \includegraphics[width=0.98\linewidth]{funsd_params.pdf}
%     \vspace{-2mm}
%     \caption{\textbf{Model Size vs. Entity Extraction F1 Score} on FUNSD benchmark. The proposed FormNets significantly outperform other approaches -- FormNet-A2 achieves higher F1 score while using a 2.5x smaller model and 7.1x less pre-training data than DocFormer~\cite{appalaraju2021docformer}.}
%     \label{fig:funsd_params}
%     \vspace{-3mm}
% \end{figure}

\begin{figure*}[ht]
\centering
\includegraphics[width=0.98\linewidth]{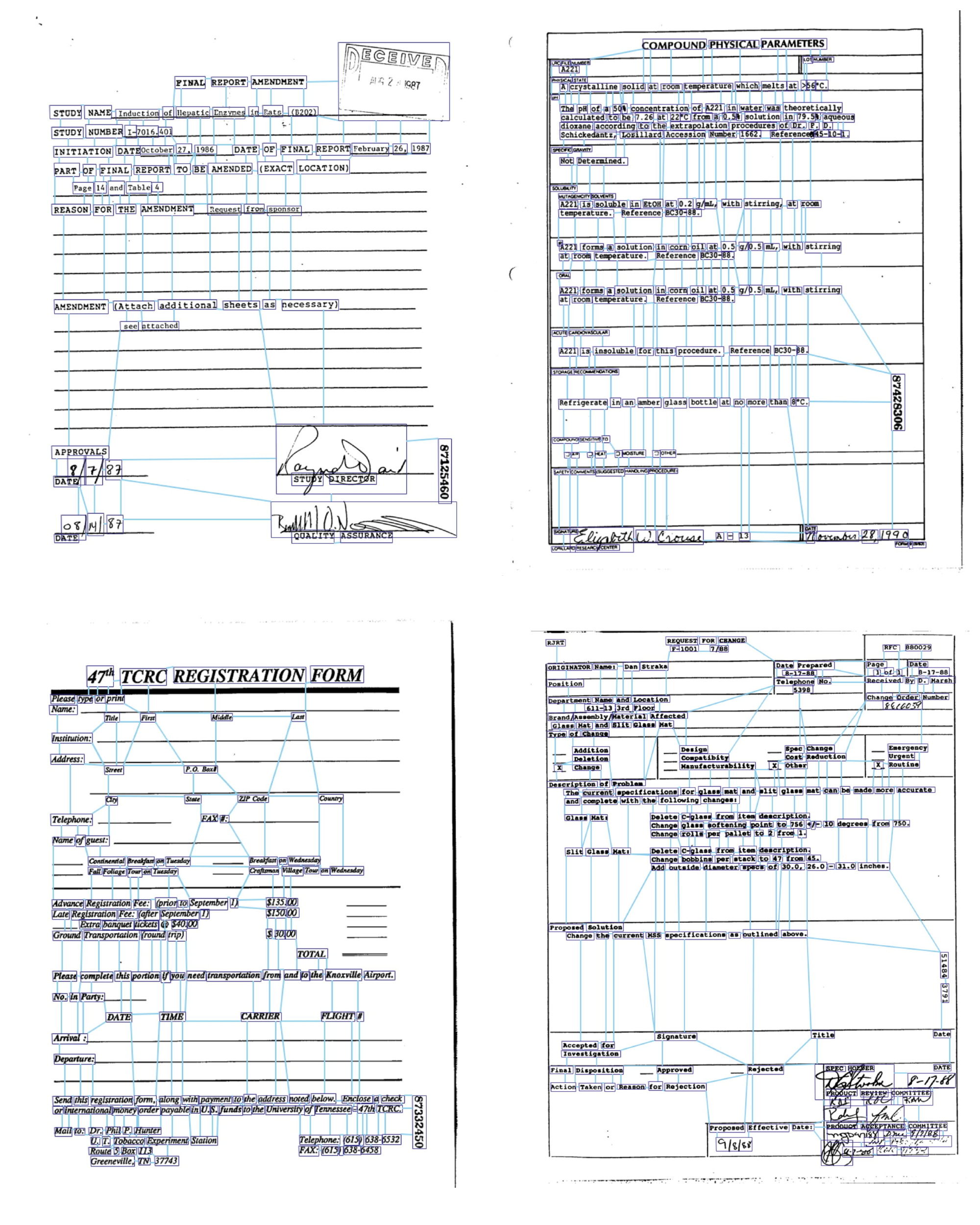}
\vspace{-2mm}
\caption{Illustrations of word-level $\beta$-skeleton graph of FUNSD documents. $\beta$-skeleton graphs provide structurally meaningful connectivity between vertices for effective message passing during representation learning and inference.}
\label{fig:beta_skeleton}
\end{figure*}

\begin{figure*}[ht]
    \centering
    \includegraphics[width=0.9\linewidth]{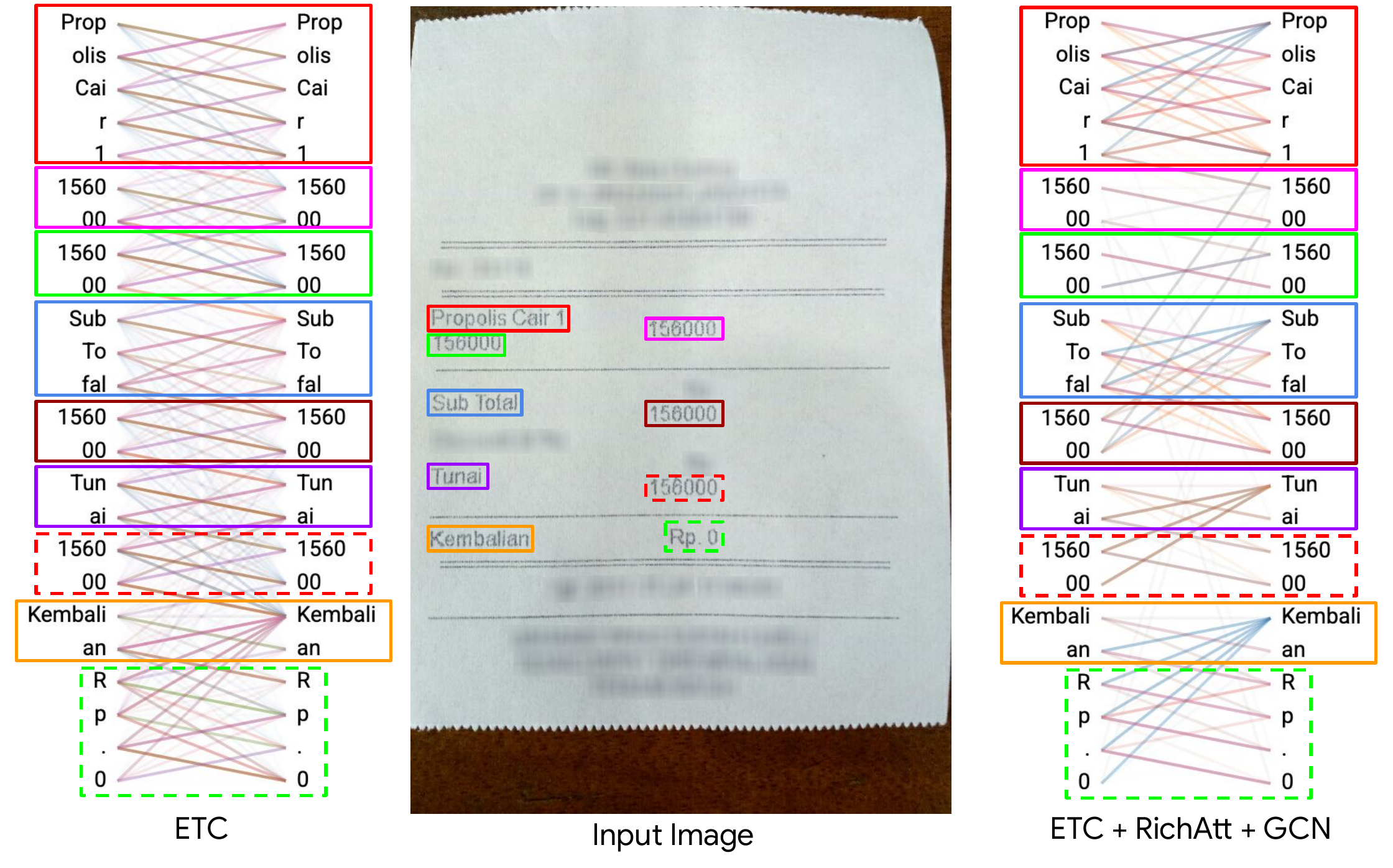}
    \includegraphics[width=0.94\linewidth]{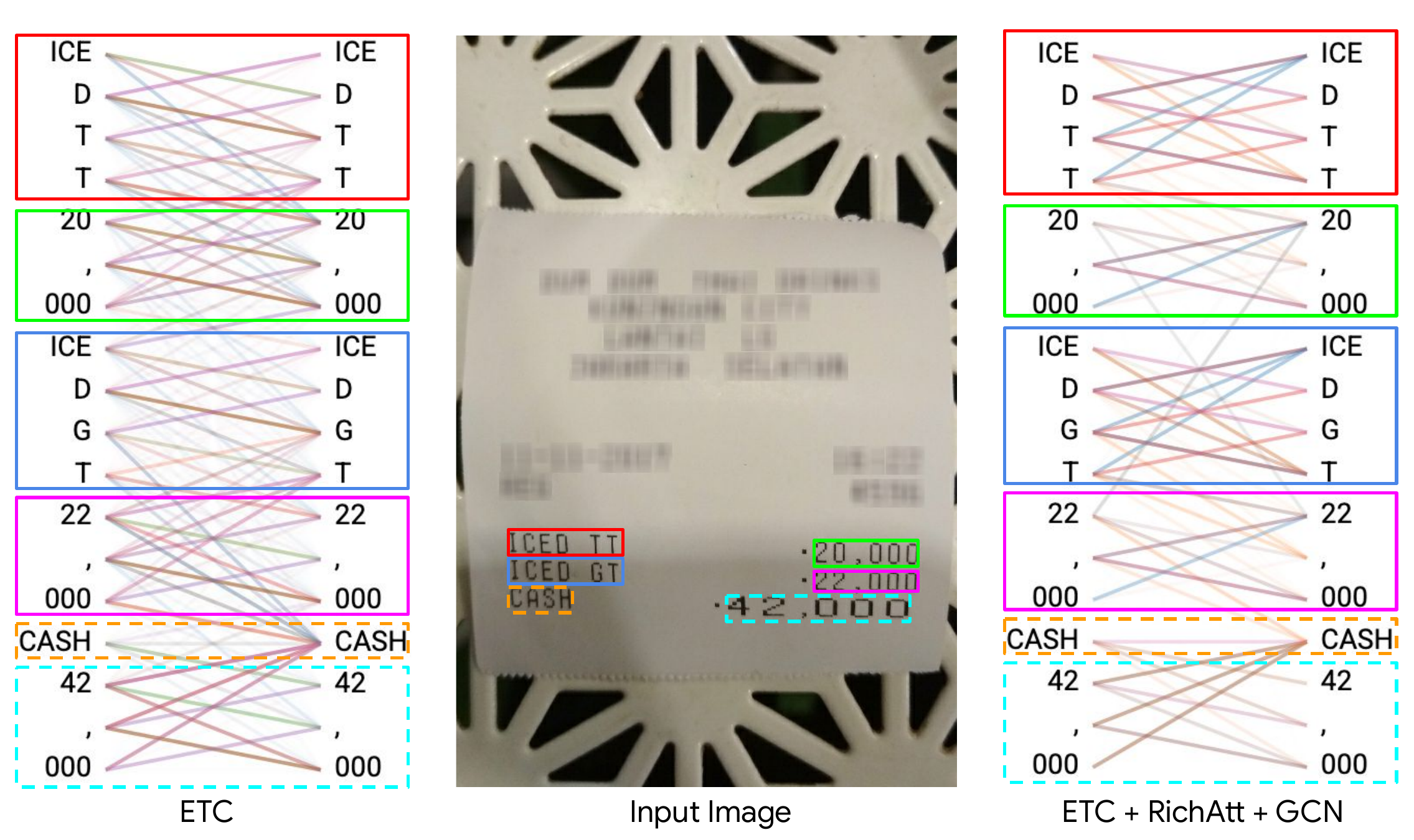}
    \vspace{-2mm}
    \caption{Examples of attention scores on CORD documents for ETC and ETC+RichAtt+GCN models. Unlike the ETC model, the ETC+RichAtt+GCN model makes tokens attend to other tokens within the same visual blocks, along with tokens aligned horizontally, thus strongly leveraging structural cues.}
    \label{fig:visualization2}
    \vspace{-2mm}
\end{figure*}

\begin{figure*}
% We need a minipage to make the footnote work inside the figure.
\begin{minipage}{\textwidth}
\begin{theorem}\label{theorem: h given a}
If $\mathbf{h}_i,\mathbf{h}_j \mid a_{ij}$ are normally distributed, then $L(\mathbf{h}_i,\mathbf{h}_j \mid a_{ij}; \mu, \Sigma) \approx \mathbf{q}_i^\top\mathbf{k}_j + c_i$, where $\mathbf{q}_i = \text{affine}^{(q)}(\mathbf{h}_i)$ and $\mathbf{k}_j = \text{affine}^{(k)}(\mathbf{h}_j)$.
\end{theorem}
\begin{proof}
We partition the parameters of the normal distribution into blocks. Because the covariance matrix is positive definite, the ``key-key'' block must be positive definite as well, meaning it can be decomposed into a single matrix multiplied times its transpose. We expand the probability into the Gaussian probability density function\footnote{Let $\tau = 2\pi$.} and apply the natural logarithm -- canceling out the $\exp$ function -- and put the normalization constant in a separate $c$ term.
\begin{align*}
    \mu &= \left[\begin{matrix}
        \mathbf{b}^{(q)}\\
        \mathbf{b}^{(k)}
    \end{matrix}\right]\\
    \Sigma^{-1} &= \left[\begin{matrix}
        V & W^{\top(q)}W^{(k)}\\
        W^{\top(k)}W^{(q)} & W^{\top(k)}W^{(k)}
      \end{matrix}\right]\\
    P(\mathbf{h}_i, \mathbf{h}_j \mid a_{ij}; \mu, \Sigma) &= \frac{1}{\sqrt{\tau^{2d}|\Sigma|}}\exp\left( -\frac{1}{2}([\mathbf{h}_i;\mathbf{h}_j] - \mu)^\top\Sigma^{-1}([\mathbf{h}_i;\mathbf{h}_j] - \mu)\right)\\
    L(\mathbf{h}_i, \mathbf{h}_j \mid a_{ij}; \mu, \Sigma) &= c -\frac{1}{2}([\mathbf{h}_i;\mathbf{h}_j] - \mu)^\top\Sigma^{-1}([\mathbf{h}_i;\mathbf{h}_j] - \mu)
\end{align*}
From here, we distribute the bilinear transformation and simplify. The result is dot product attention with an ignoreable constant $c_i$ (because it later divides out of the softmax function) and an extra bias term composed from the ``key'' representation alone. We do not explore the effect of this newly-derived bias term in this work.
\begin{align*}
    L(\mathbf{h}_i, \mathbf{h}_j \mid a_{ij}) &= c -\frac{1}{2}\left(\mathbf{h}_i - \mathbf{b}^{(q)}\right)^\top V\left(\mathbf{h}_i - \mathbf{b}^{(q)}\right)\\
        &\qquad + \left(\mathbf{h}_i - \mathbf{b}^{(q)}\right)^\top W^{\top(q)}W^{(k)}\left(\mathbf{h}_j - \mathbf{b}^{(k)}\right)\\
        &\qquad -\frac{1}{2}\left(\mathbf{h}_j - \mathbf{b}^{(k)}\right)^\top W^{\top(k)}W^{(k)}\left(\mathbf{h}_j - \mathbf{b}^{(k)}\right)\\
    &= c_i + \left(\mathbf{h}_i - \mathbf{b}^{(q)}\right)^\top W^{\top(q)}W^{(k)}\left(\mathbf{h}_j - \mathbf{b}^{(k)}\right)\\
        &\qquad -\frac{1}{2}\left(\mathbf{h}_j - \mathbf{b}^{(k)}\right)^\top W^{\top(k)}W^{(k)}\left(\mathbf{h}_j - \mathbf{b}^{(k)}\right)\\
    &= c_i + \left(W^{(q)}\mathbf{h}_i + W^{(q)}\mathbf{b}^{(q)}\right)^\top\left(W^{(k)}\mathbf{h}_j + W^{(k)}\mathbf{b}^{(k)}\right)\\
        &\qquad - \frac{1}{2}\left(W^{(k)}\mathbf{h}_j + W^{(k)}\mathbf{b}^{(k)}\right)^\top\left(W^{(k)}\mathbf{h}_j + W^{(k)}\mathbf{b}^{(k)}\right)\\
    \mathbf{q}_i &= \text{affine}^{(q)}(\mathbf{h}_i)\\
    \mathbf{k}_j &= \text{affine}^{(k)}(\mathbf{h}_j)\\
    L(\mathbf{h}_i, \mathbf{h}_j \mid a_{ij}) &= c_i + \mathbf{q}_i^\top\mathbf{k}_j -\frac{1}{2}\mathbf{k}_j^\top\mathbf{k}_j\\
    &\approx c_i + \mathbf{q}_i^\top\mathbf{k}_j
\end{align*}
\end{proof}
\end{minipage}
\end{figure*}

\begin{figure*}
\begin{lemma}\label{lemma: x given f}
If $\mathbf{x} \mid f$ is normally distributed, and $f$ is categorical, then $L(\mathbf{x} \mid f = y; \mu_y, \Sigma_y) = \text{biaffine}_y(\mathbf{x})$.
\end{lemma}
\begin{proof}
This can be shown by simply expanding out the probability density function for the normal distribution -- with parameters specific to the value $f$ takes -- and simplifying. The matrix $V_y$ in the result must be negative definite, but this is of little consequence for what follows.
\begin{align*}
    P(\mathbf{x} \mid f = y; \mu_y, \Sigma_y) &= \frac{1}{\sqrt{\tau^d|\Sigma_y|}}\exp\left(-\frac{1}{2}(\mathbf{x} - \mu_y)^\top\Sigma_y^{-1}(\mathbf{x} - \mu_y)\right)\\
    L(\mathbf{x} \mid f) &= \mathbf{x}^\top\left(-\frac{1}{2}\Sigma_y^{-1}\right)\mathbf{x} + \left(\mu_y^\top\Sigma_y^{-1}\right)\mathbf{x} + \left(-\frac{1}{2}\mu_y^\top\Sigma_y^{-1}\mu_y - \frac{1}{2}\ln\big(\tau^d |\Sigma_y|\big)\right)\\
    &= \mathbf{x}^\top V_y \mathbf{x} + \mathbf{w}_y^\top\mathbf{x} + b_y\\
    &= \text{biaffine}_y(\mathbf{x})
\end{align*}
\end{proof}
\end{figure*}

\begin{figure*}
\begin{lemma}\label{lemma: f given x}
If $f \mid \mathbf{x}$ is Bernoulli-distributed and $\mathbf{x} \mid f$ is normally distributed, then $P(f \mid \mathbf{x}; p, \mu_0, \mu_1, \Sigma_0, \Sigma_1) = \text{sigmoid}(\text{biaffine}(\mathbf{x}))$, and if $\Sigma_0 = \Sigma_1$, then $ = \text{sigmoid}(\text{affine}(\mathbf{x}))$.
\end{lemma}
\begin{proof}
We begin by applying Bayes' rule and exponentiating by the log in order to express the probability in terms of the sigmoid function. Then we apply Lemma \ref{lemma: x given f} to expand out the log-likelihood terms, and we use the Bernoulli probability mass function to expand out the log-prior term. This results in the sum of multiple biaffine and constant terms, which is equivalent to a single biaffine function.
\begin{align*}
    P(f \mid \mathbf{x}; \mu_0, \mu_1, \Sigma_0, \Sigma_1) &= \frac{P(\mathbf{x} \mid f; \mu_1, \Sigma_1)P(f; p)}{P(\mathbf{x} \mid f; \mu_1, \Sigma_1)P(f; p) + P(\mathbf{x} \mid \neg f; \mu_0, \Sigma_0)P(\neg f; 1 - p))}\\
    &=\frac{\exp\left(L(\mathbf{x} \mid f) + L(f)\right)}{\exp\left(L(\mathbf{x} \mid f) + L(f)\right) + \exp\left(L(\mathbf{x} \mid \neg f) + L(\neg f)\right)}\\
    &= \text{sigmoid}(L(\mathbf{x} \mid \neg f) + L(\neg f) - (L(\mathbf{x} \mid f) + L(f)))\\
    &=\text{sigmoid}\left(\text{biaffine}_0(\mathbf{x}) + \ln(1-p) - (\text{biaffine}_1 + \ln(p))\right)\\
    &=\text{sigmoid}\left(\text{biaffine}(\mathbf{x})\right)
\end{align*}
Recall from Lemma \ref{lemma: x given f} that the bilinear term $V_y$ of the biaffine function is just $-\frac{1}{2}\Sigma_y^{-1}$, independent of $\mu_y$. Therefore if $\Sigma_0 = \Sigma_1$, then $V_0 - V_1 = 0$, and the two bilinear terms cancel out when simplifying $\text{biaffine}_0 - \text{biaffine}_1$. Thus in this context, the biaffine function reduces to an affine function.
\begin{align*}
    P(f \mid \mathbf{x}; p, \mu_0, \mu_1, \Sigma_0, \Sigma_1) &= \text{sigmoid}\left(\text{affine}(\mathbf{x})\right) & \text{ if }\Sigma_0 = \Sigma_1
\end{align*}
\end{proof}
\end{figure*}

\begin{figure*}
\begin{theorem}\label{theorem: binary f}
If $f_{ij} \mid \mathbf{h}_i, \mathbf{h}_j, a_{ij}$ is Bernoulli-distributed, and $\mathbf{h}_i, \mathbf{h}_j \mid f_{ij}, a_{ij}$ is normally distributed, and $\Sigma_0 = \Sigma_1$ (as in Lemma \ref{lemma: f given x}), then $L(f_{ij} = y \mid \mathbf{h}_i, \mathbf{h}_j, a_{ij}; \theta) = y\ln(\text{sigmoid}(\text{affine}([\mathbf{h}_i; \mathbf{h}_j]))) + (1 - y)\ln(1 - \text{sigmoid}(\text{affine}([\mathbf{h}_i; \mathbf{h}_j])))$.
\end{theorem}
\begin{proof}
Because $f_{ij}$ is binary, we begin by expressing the probability mass function in terms of both $f_{ij}$ and $\neg f_{ij}$. Then we apply Theorem \ref{lemma: f given x} to replace the abstract probability term with a fully-specified parametric function. Finally, the natural log can be applied and simplified straightforwardly.
\begin{align*}
    P(f_{ij} = y \mid \mathbf{h}_i, \mathbf{h}_j, a_{ij}; \theta) &= P(f_{ij} \mid \mathbf{h}_i, \mathbf{h}_j, a_{ij})^y(1 - P(f_{ij} \mid \mathbf{h}_i, \mathbf{h}_j, a_{ij}))^{1 - y}\\
    &= \text{sigmoid}(\text{affine}([\mathbf{h}_i; \mathbf{h}_j]))^y(1 - \text{sigmoid}(\text{affine}([\mathbf{h}_i; \mathbf{h}_j])))^{1 - y}\\
    L(f_{ij} = y \mid \mathbf{h}_i, \mathbf{h}_j, a_{ij}; \theta)&= y\ln(\text{sigmoid}(\text{affine}([\mathbf{h}_i; \mathbf{h}_j])))\\
        &\qquad+ (1 - y)\ln(1 - \text{sigmoid}(\text{affine}([\mathbf{h}_i; \mathbf{h}_j])))
\end{align*}
\end{proof}
\end{figure*}

\begin{figure*}
\begin{lemma}\label{lemma: log-normal}
The log-normal probability density function can be written as $\frac{1}{\sqrt{\tau\sigma^2\exp(\sigma^2)}}\exp\left(-\frac{(\ln(x) - \mu')^2}{2\sigma^2} - \mu'\right)$, where $\mu' = \mu - \sigma^2$.
\end{lemma}
\begin{proof}
This can be shown through basic algebra.
\begin{align*}
    \frac{1}{x\sqrt{\tau\sigma^2}}\exp\left(-\frac{(\ln(x) - \mu)^2}{2\sigma^2}\right) &= \frac{1}{\sqrt{\tau\sigma^2}}\exp\left(-\frac{(\ln(x) - \mu)^2}{2\sigma^2} - \ln(x)\right)\\
    &= \frac{1}{\sqrt{\tau\sigma^2}}\exp\left(-\frac{(\ln(x) - \mu)^2}{2\sigma^2} - \frac{2\sigma^2\ln(x)}{2\sigma^2}\right)\\
    &= \frac{1}{\sqrt{\tau\sigma^2}}\exp\left(-\frac{\ln(x)^2 - 2(\mu - \sigma^2)\ln(x) + \mu^2}{2\sigma^2}\right)\\
    &= \frac{1}{\sqrt{\tau\sigma^2}}\exp\left(-\frac{(\ln(x) - (\mu - \sigma^2))^2 + \sigma^2(2\mu - \sigma^2)}{2\sigma^2}\right)\\
    &= \frac{1}{\sqrt{\tau\sigma^2}}\exp\left(-\frac{(\ln(x) - (\mu - \sigma^2))^2}{2\sigma^2} - \mu + \frac{\sigma^2}{2}\right)\\
    &= \frac{1}{\sqrt{\tau\sigma^2}}\exp\left(-\frac{(\ln(x) - \mu')^2}{2\sigma^2} - \mu' - \frac{\sigma^2}{2}\right)\\
    &= \frac{1}{\sqrt{\tau\sigma^2\exp(\sigma^2)}}\exp\left(-\frac{(\ln(x) - \mu')^2}{2\sigma^2} - \mu'\right)
\end{align*}
\end{proof}
\end{figure*}

\begin{figure*}
\begin{theorem}\label{theorem: exp f}
If $\ln(f_{ij}), \mathbf{h}_i, \mathbf{h}_j \mid a_{ij}$ is normally distributed, then $L(f_{ij} = z \mid \mathbf{h}_i, \mathbf{h}_j, a_{ij}; \Sigma, \mu) = -\theta^2(\ln(z) - \text{affine}([\mathbf{h}_i; \mathbf{h}_j]))^2 / 2 - \text{affine}([\mathbf{h}_i; \mathbf{h}_j])$.
\end{theorem}
\begin{proof}
For convenience and brevity, we stack $\mathbf{h}_i, \mathbf{h}_j$ into one vector $\mathbf{h}_{ij}$. As usual, we apply the probability density function of the assumed probability distribution. Note that here we begin with the \emph{joint} normal distribution; this allows us to avoid complexities arising from mixing normally- and lognormally-distributed variables. 
\begin{align*}
    \mathbf{h}_{ij} &= [\mathbf{h}_i; \mathbf{h}_j]\\
    P(\ln(f_{ij}) = z, \mathbf{h}_{ij} \mid a_{ij}; \mu, \Sigma) &= \frac{1}{\sqrt{\tau^{2d+1}|\Sigma|}}\exp\left(([z; \mathbf{h}_{ij}] - \mu)^\top\Sigma^{-1}([z; \mathbf{h}_{ij}] - \mu)\right)
\end{align*}
Similar to Theorem \ref{theorem: h given a}, we partition the parameters of the normal distribution into one section for the log-normal feature $\ln(f)$ and one section for the normal features $\mathbf{h}_{ij}$. Then we apply the definition of a \emph{conditional} normal distribution as described by \citet{eaton1983multivariate} to get the distribution of the new feature conditioned on $\mathbf{h}_{ij}$. 
\begin{align*}
    \mu &= \left[\begin{matrix}
        b^{(f)}\\
        \mathbf{b}^{(h)}
    \end{matrix}\right]\\
    \Sigma &= \left[\begin{matrix}
        w^{(ff)} & \mathbf{w}^{\top(hf)}\\
        \mathbf{w}^{(hf)} & W^{(hh)}
    \end{matrix}\right]\\
    \mu' &= b^{(f)} + \mathbf{w}^{\top(hf)}W^{-1(hh)}(\mathbf{h}_{ij} - \mathbf{b}^{(h)})\\
    \sigma^{2\prime} &= w^{(ff)} - \mathbf{w}^{\top(hf)}W^{-1(hh)}\mathbf{w}^{(hf)}\\
    P(\ln(f_{ij}) = z \mid \mathbf{h}_{ij}, a_{ij}; \mu', \sigma^{2\prime}) &= \frac{1}{\sqrt{\tau\sigma^{2\prime}}}\exp\left(-\frac{(z - \mu')^2}{2\sigma^{2\prime}}\right)
\end{align*}
We convert the normal distribution over $\ln(f_{ij})$ into a log-normal distribution over $f_{ij}$ in the convenient form derived in Lemma \ref{lemma: log-normal}, and simplify its log-probability. Noting that $\mu''$ is ultimately an affine function of $\mathbf{h}_{ij}$, whereas $\sigma^{2\prime}$ is composed exclusively from free parameters, we replace the former with an affine function and the latter with a constant $1 / \theta^2$ (for better numerical stability under gradient-based optimization).
\begin{align*}
    P(f_{ij} = z \mid \mathbf{h}_{ij}, a_{ij}; \mu'', \sigma^{2\prime}) &= \frac{1}{\sqrt{\tau\sigma^{2\prime}\exp(\sigma^{2\prime})}}\exp\left(-\frac{(\ln(z) - \mu'')^2}{2\sigma^{2\prime}} - \mu''\right)\\
    L(f_{ij} = z \mid \mathbf{h}_{ij}, a_{ij}; \mu'', \sigma^{2\prime}) &= c - \frac{(\ln(z) - \mu'')^2}{2\sigma^{2\prime}} - \mu''\\
    &= c -\frac{\theta^2(\ln(z) - \text{affine}(\mathbf{h}_{ij}))^2}{2} - \text{affine}(\mathbf{h}_{ij})
\end{align*}
Note that when implementing attention in a neural network, the second instance of the affine term can be absorbed into the affine components of dot-product attention and ignored.
\end{proof}
\end{figure*}

\end{document}